\documentclass{article}

\usepackage{microtype}
\usepackage{graphicx}
\usepackage{subcaption}
\usepackage{booktabs} % for professional tables
\usepackage{amsmath}
\usepackage{multirow}

\usepackage{hyperref}

\usepackage[accepted]{icml2026}

\usepackage{amsmath}
\usepackage{amssymb}
\usepackage{mathtools}
\usepackage{amsthm}
\usepackage[capitalize,noabbrev]{cleveref}

\theoremstyle{plain}
\newtheorem{theorem}{Theorem}[section]
\newtheorem{proposition}[theorem]{Proposition}
\newtheorem{lemma}[theorem]{Lemma}

\theoremstyle{definition}

\theoremstyle{remark}

\newcommand{\re}[1]{\textcolor{black}{#1}}

\usepackage[textsize=tiny]{todonotes}

\icmltitlerunning{GASS for Disentangled Diversity Enhancement in T2I Generation}

\begin{document}

\twocolumn[
  \icmltitle{GASS: Geometry-Aware Spherical Sampling for \\ Disentangled Diversity Enhancement in Text-to-Image Generation}

  % It is OKAY to include author information, even for blind submissions: the
  % style file will automatically remove it for you unless you've provided
  % the [accepted] option to the icml2026 package.

  % List of affiliations: The first argument should be a (short) identifier you
  % will use later to specify author affiliations Academic affiliations
  % should list Department, University, City, Region, Country Industry
  % affiliations should list Company, City, Region, Country

  % You can specify symbols, otherwise they are numbered in order. Ideally, you
  % should not use this facility. Affiliations will be numbered in order of
  % appearance and this is the preferred way.
  \icmlsetsymbol{equal}{*}

  \begin{icmlauthorlist}
    \icmlauthor{Ye Zhu}{lix,pu}
    \icmlauthor{Kaleb S. Newman}{pu}
    \icmlauthor{Johannes F. Lutzeyer}{lix} \\
    \icmlauthor{Adriana Romero-Soriano}{meta,mcg,mila,cifar}
    \icmlauthor{Michal Drozdzal}{meta}
    \icmlauthor{Olga Russakovsky}{pu}
    % \icmlauthor{Firstname7 Lastname7}{comp}
    %\icmlauthor{}{sch}
    % \icmlauthor{Firstname8 Lastname8}{sch}
    % \icmlauthor{Firstname8 Lastname8}{yyy,comp}
    %\icmlauthor{}{sch}
    %\icmlauthor{}{sch}
  \end{icmlauthorlist}

  \icmlaffiliation{lix}{Laboratoire d'Informatique (LIX), CNRS, École Polytechnique, IPP, France}
  \icmlaffiliation{pu}{Department of Computer Science, Princeton University, USA}
  \icmlaffiliation{meta}{FAIR at Meta - Montreal, Canada}
  \icmlaffiliation{mcg}{McGill University, Canada}
  \icmlaffiliation{mila}{Mila, Quebec AI Institute, Canada}
  \icmlaffiliation{cifar}{Canada CIFAR AI chair}

  \icmlcorrespondingauthor{Ye Zhu}{ye.zhu@polytechnique.edu}
  % \icmlcorrespondingauthor{Firstname2 Lastname2}{first2.last2@www.uk}

  % You may provide any keywords that you find helpful for describing your
  % paper; these are used to populate the "keywords" metadata in the PDF but
  % will not be shown in the document
  \icmlkeywords{Machine Learning, ICML}

  \vskip 0.3in
]

% this must go after the closing bracket ] following \twocolumn[ ...

% This command actually creates the footnote in the first column listing the
% affiliations and the copyright notice. The command takes one argument, which
% is text to display at the start of the footnote. The \icmlEqualContribution
% command is standard text for equal contribution. Remove it (just {}) if you
% do not need this facility.

% Use ONE of the following lines. DO NOT remove the command.
% If you have no special notice, KEEP empty braces:
\printAffiliationsAndNotice{}  % no special notice (required even if empty)
% Or, if applicable, use the standard equal contribution text:
% \printAffiliationsAndNotice{\icmlEqualContribution}

\begin{abstract}

Despite high semantic alignment, modern text-to-image (T2I) generative models still struggle to synthesize diverse images from a given prompt. 
% This lack of diversity not only restricts user choice, but also risks amplifying societal biases. 
In this work, we enhance the T2I diversity through a geometric lens. Unlike most existing methods that rely primarily on entropy-based guidance to increase sample dissimilarity, we introduce \textbf{G}eometry-\textbf{A}ware \textbf{S}pherical \textbf{S}ampling (\emph{GASS}) to enhance diversity by explicitly controlling both prompt-dependent and prompt-independent sources of variation. Specifically, we decompose the diversity measure in CLIP embeddings using two orthogonal directions: the text embedding, which captures semantic variation related to the prompt, and an identified orthogonal direction that captures prompt-independent variation (e.g., backgrounds). Based on this decomposition, \emph{GASS} increases the geometric projection spread of generated image embeddings along both axes and guides the T2I sampling process via expanded predictions along the generation trajectory. Our experiments on different frozen T2I backbones (U-Net and DiT, diffusion and flow) and benchmarks demonstrate the effectiveness of disentangled diversity enhancement with minimal impact on image fidelity and semantic alignment~\footnote{Code is available at \href{https://github.com/L-YeZhu/GASS_T2I}{https://github.com/L-YeZhu/GASS\_T2I}.}.
% Code is available at \href{}{}

% Despite high semantic alignment, modern text-to-image (T2I) generative models still struggle to synthesize diverse images from a single text prompt. This limitation not only restricts user choice and creative control but also risks amplifying societal biases by reinforcing narrow visual stereotypes. In this work, we propose enhancing output diversity in T2I generation through a geometric lens. Unlike most existing methods that rely primarily on entropy-based guidance to increase sample dissimilarity, we introduce Geometry-Aware Spherical Sampling (GASS), which amplifies diversity by explicitly controlling both prompt-dependent and prompt-independent sources of variation. Specifically, we decompose the diversity space in CLIP embeddings using two orthogonal directions: the text embedding, which captures semantic variation related to the prompt, and an identified orthogonal direction that captures prompt-independent variation, such as backgrounds. Based on this decomposition, GASS increases the geometric spread of generated image embeddings along both axes and guides the T2I sampling process via perturbed predictions throughout the generation trajectory. Our experiments on different T2I backbones (U-Net and DiT, diffusion and flow) and benchmarks demonstrate not only the effectiveness of diversity enhancement but also introduce an additional layer of controlability in the diversity amplification task through our geometrical disentanglement.

\end{abstract}

\vspace{-0.15in}
\section{Introduction}
\label{sec:intro}

\begin{figure}
    \centering
    \includegraphics[width=1.0\linewidth]{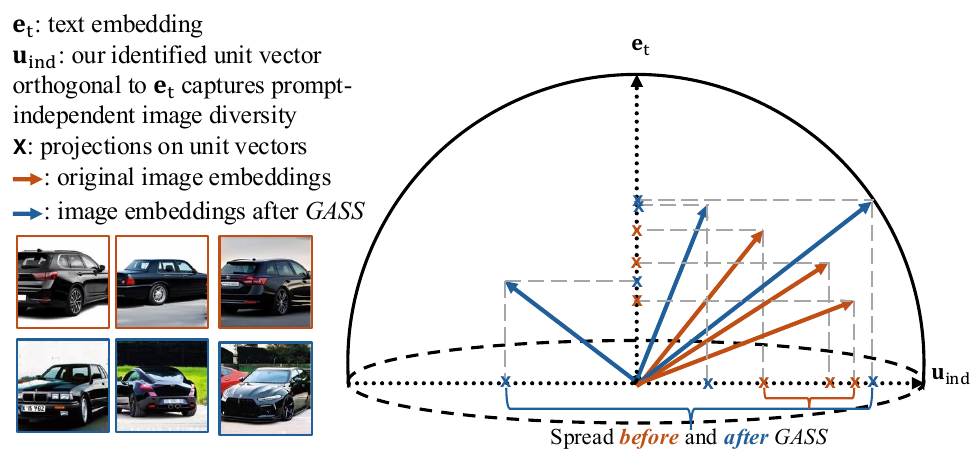}
    \caption{\textbf{Illustration of our geometric decomposition of sample diversity and \emph{
    GASS} enhancement method in CLIP space.} We decompose the diversity of generated image batches from T2I models in the CLIP hypersphere along two orthogonal axes: text embedding $\mathbf{e}_t$ (i.e., prompt-dependent) and our identified direction $\mathbf{u}_{\text{ind}}$ (i.e., prompt-independent). Our \emph{GASS} method explicitly expands the geometric spread along both axes, thus enhancing the diversity of generated images across prompt-dependent content (e.g., object viewing angles) and prompt-independent visual attributes (e.g., backgrounds).}
    \label{fig1:idea}
    % \vspace{-0.15in}
\end{figure}

%Para 1: Diversity problem in T2I
Text-to-Image (T2I) generation has gained tremendous popularity and research attention in recent years, driven by advances in model design, including diffusion-based~\cite{ho2020denoising,song2021score} and flow-based~\cite{papamakarios2021normalizing,liu2023rectifiedflow,lipmanflow} architectures, as well as successful scaling on large-scale text-image datasets~\cite{rombach2022stabled,esser2024sd3}. However, despite significant improvements in image fidelity and semantic alignment with text conditions, these models still tend to generate images with limited diversity given a fixed text prompt. 
This lack of diversity creates practical and societal challenges. It restricts not only the user choice and creative control in generative design workflows, but also risks amplifying societal biases by reinforcing narrow visual stereotypes related to attributes such as gender and ethnicity~\cite{naik2023social,wan2024survey}. To address this, we seek to enhance the diversity of generated images within the T2I context under a fixed text prompt in this work.

%Limitation in current diversity measure and enhancement techniques
Prior work often investigates the diversity challenge through the design of evaluation and enhancement methods. 
From the evaluation perspective, diversity is typically assessed either in a reference-based manner by comparing generated samples against real images using distributional or coverage metrics~\cite{kynkaanniemi2019improvedprecision,naeem2020reliable}, or in a reference-free manner through quantifying entropy directly in the embedding space~\cite{friedmanvendi2023,ospanov2025scendi}.
As for enhancement techniques, recent methods typically improve diversity by maximizing sample dissimilarity within the batch through perturbations to intermediate latents or conditioning signals~\cite{sadatcads,pg2024,kirchhof2025spell,berrada2025entropy}, aligning with metrics like Vendi Score (VS)~\cite{friedmanvendi2023} that measure embedding entropy.
However, these entropy-maximization approaches overlook the multi-sourced nature of T2I diversity. For instance, given a prompt like \emph{``A black colored car’’}, outputs vary across prompt-dependent dimensions (e.g., viewing angles, car models) and prompt-independent dimensions (e.g., backgrounds, lighting). While recent Scendi scores attempt decomposition via Schur complement entropy~\cite{ospanov2025scendi,jalalisparke}, their reliance on text-image covariance matrices limits applicability to scenarios with equal numbers of prompts and images.
Unlike entropy-based approaches, we address this challenge by disentangling and quantifying these sources of variation from a geometric perspective.

We consider the scenario of generating multiple images from a single prompt, and propose to analyze their diversity within the shared CLIP embedding hypersphere~\cite{radford2021clip}. 
As illustrated in Fig.~\ref{fig1:idea}, we decompose the variation of generated image embeddings $\mathbf{V} = \{\mathbf{e}_i\}_{i=1}^B$ relative to the text embedding $\mathbf{e}_t$ into two orthogonal components: the \textbf{prompt-dependent variation} captured by their projections onto $\mathbf{e}_t$, which represent semantic changes aligned with the text condition; and the \textbf{prompt-independent variation} captured by our identified orthogonal complement $\mathbf{u}_{\text{ind}}$, featuring visual attributes like backgrounds and styles.
We further propose to quantify the diversity of the image batch by \emph{summing the respective projection spreads} along each direction. Empirically, we validate this measurement on ImageNet~\cite{deng2009imagenet,russakovsky2015imagenet} by comparing the geometric spread of real images against synthetic generations from T2I models, as detailed in Sec.~\ref{sec:measure}.

Building upon this geometric analysis, we propose the \textbf{G}eometry-\textbf{A}ware \textbf{S}pherical \textbf{S}ampling (\emph{GASS}) method to enhance the generated sample diversity within the T2I setting given a fixed text prompt. Specifically, \emph{GASS} explicitly expands the projection spread of generated embeddings along both orthogonal directions, as illustrated in Fig.~\ref{fig1:idea}. We then update the images through gradient-based optimization using the frozen CLIP image encoder. These optimized images are used to replace predicted images within the T2I sampling process, thus steering the generation trajectory toward greater geometric coverage while preserving semantic fidelity. Extensive experiments across diverse T2I backbones (U-Net~\cite{ronneberger2015unet} and DiT architectures~\cite{peebles2023dit}, diffusion~\cite{rombach2022stabled} and flow paradigms~\cite{esser2024sd3}) and benchmarks (ImageNet~\cite{russakovsky2015imagenet} and DrawBench~\cite{saharia2022photorealistic}) demonstrate that \emph{GASS} achieves superior diversity gains compared to state-of-the-art enhancement techniques while maintaining quality and consistency.
Notably, to the best of our knowledge, \emph{GASS} is the first sampling-based method to explicitly introduce \textbf{meaningful background diversity without modifying text prompts}, as shown in the \textbf{non-cherry-picked results} from Fig.~\ref{fig:qualitative}. This suggests that our geometric formulation enables more comprehensive exploration of the residual space, which has remained largely unexploited in prior work.
% \re{Code is available at \href{https://github.com/L-YeZhu/GASS_T2I}{https://github.com/L-YeZhu/GASS\_T2I}.}

% Based on this decomposition, we further propose our \textbf{G}eometry-\textbf{A}ware \textbf{S}pherical \textbf{S}ampling (\emph{GASS}) method to explicitly expand the  

Our contributions can be summarized as follows:
\begin{itemize}
    \item We introduce a geometric framework to disentangle and quantify prompt-dependent and prompt-independent diversity sources within the CLIP hypersphere for T2I generation.
    \item We propose \emph{GASS}, a geometry-aware spherical sampling method that enhances diversity by explicitly expanding the geometric spread of generated embeddings along orthogonal directions.
    \item Extensive experiments across diverse T2I backbones and benchmarks demonstrate the effectiveness of \emph{GASS} for disentangled diversity enhancement.
\end{itemize}

\paragraph{Conflict of Interest Disclosure.}
\re{The authors declare that there are no conflicts of interest.}

%Contributions: Spherical spread and enhancement techniques
% The contribution of our work can be summarized as follows:
% \begin{itemize}
    % \item We introduce a geometric disentanglement method and conduct analysis in the text-visual mutual latent space (e.g. CLIP) to identify and quantify different sources of diversity for T2I generation.
    % \item We propose \emph{GASS}, a geometry-aware spherical sampling method, to enhance the diversity of the generated image set given a text prompt by increasing their spherical spread in the latent space.
    % \item We perform extensive experiments across different T2I models and benchmarks to demonstrate the effectiveness of our proposed diversity measure and sampling method. 
% \end{itemize}

\section{Related Work}
\label{sec:related}

\subsection{Diversity Evaluation and Measurement in T2I}
\label{subsec:div_measurement}
Beyond commonly adopted quality and alignment assessment through scores such as FID~\cite{heusel2017fid} and CLIPScore~\cite{hessel2021clipscore}, sample diversity in T2I remains a critical yet challenging axis of evaluation. Existing assessments can be broadly categorized into reference-based metrics~\cite{kynkaanniemi2019improvedprecision,naeem2020density}, which rely on ground-truth data distributions, and reference-free metrics that assess intrinsic sample variety~\cite{friedmanvendi2023,pasarkar2024cousins,jalali2024conditional,ospanovvendi,ospanov2024towards}.
Specifically, Precision and Recall~\cite{kynkaanniemi2019improvedprecision}, Density and Coverage~\cite{naeem2020density} are two classic score pairs that simultaneously capture the sample quality and diversity by measuring the distributional overlap between generated samples and real reference data. 
Image Retrieval Score (IRS)~\cite{Dombrowski_2025_CVPR} is a recently introduced diversity score defined through a retrieval task.
In contrast, reference-free metrics assess diversity solely from generated samples. For instance, VS and its recent variants~\cite{friedmanvendi2023,pasarkar2024cousins,jalali2024conditional,ospanov2024towards,ospanov2025scendi} quantify intrinsic diversity via the entropy of the sample similarity matrix.
Our work also looks into the diversity in a reference-free manner, by decomposing the batch of image CLIP embeddings into prompt-dependent and prompt-independent components through geometrically grounded orthogonal projection on the high-dimensional unit sphere. 

%%%%%% Arxived previous version%%%%%%
% Evaluating text-to-image (T2I) generation typically involves assessing visual quality and semantic alignment, commonly measured by scores like FID~\cite{heusel2017fid} and CLIPScore~\cite{hessel2021clipscore}, respectively. Recent metrics, such as ImageReward~\cite{xu2023imagereward} and PickScore~\cite{kirstain2023pickscore}, further align evaluation with human preference. 
% More recently, the Scendi Score~\cite{ospanov2025scendi} is proposed to measure diversity by disentangling textual content from visual variety by computing the Schur Complement entropy of the paired image-text kernel covariance matrix.
% However, Scendi Score targets $N$ distinct text-image pairs and degenerates to VS in fixed-prompt settings under a non-invertible covariance matrix. 
% In contrast, our work explicitly addresses sample diversity given a fixed prompt by decomposing image CLIP embeddings into prompt-dependent and prompt-independent components via geometrically grounded orthogonal projection in the high-dimensional unit sphere. 

% Existing evaluation of sample diversity in image generation can be principally categorized into those requiring reference images or not .  

\subsection{Methods for Enhanced T2I Diversity}
\label{subsec:div_method}
Many recent research works aim to enhance generation diversity by introducing additional guidelines in various settings~\cite{ho2022cfg,sadatcads,miao2024training,cideron2024diversity,pg2024,askari2024improving,kirchhof2025spell,dall2025increasing,berrada2025entropy,kynkaanniemi2024ig}. These efforts can be broadly categorized into post-training-based and inference-time sampling-based approaches.
While some works~\cite{miao2024training,cideron2024diversity} employ RL-based reward functions during training, a larger body of work~\cite{ho2022cfg,sadatcads,kirchhof2025spell,kynkaanniemi2024ig,pg2024} focuses on inference-time guidance. Among the latter, a standard paradigm is to explicitly maximize the entropy of the generated samples~\cite{berrada2025entropy,pg2024,jalalisparke,askari2024improving}, directly targeting the information-theoretic definitions of diversity highlighted in Sec.~\ref{subsec:div_measurement}.
While the majority of them share a similar high-level objective of maximizing sample dissimilarity, they often lack granular control over the diversity nature. In contrast, our approach introduces \emph{a layer of controllability}, allowing us to explicitly maximize the diversity along either prompt-dependent (semantic alignment) or prompt-independent (e.g., backgrounds) axes.

% It is worth noting that several prior works have revealed the trade-off among different axis among sample quality and diversity~\cite{zhang2025intricate,astolfi2024consistency}
% Specifically, Particle Guidance (PG)~\cite{pg2024} introduces a joint particle time-evolving potential guidance that alters the sampling process to be non independent to enhance the diversity among those samples.
% CADS~\cite{sadatcads} proposes to perturb the conditions during the sampling inference. 

\subsection{Latent Space Analysis for Generative Models}
\label{subsec:latent_space}

In an orthogonal line of research, recent studies in dynamic generative models~\cite{ho2020denoising,song2021score,liu2023rectifiedflow} have also investigated the geometric structures of the intermediate latent space to unlock more fine-grained control. 
Based on structural understanding, multiple works~\cite{park2023understanding,zhu2023boundary,wang2024diffusion,wang2025silent,baumann2025continuous} introduce geometrically grounded perturbation and guidance over the sampling process for downstream tasks like image editing and personalization.
Despite these advances, the application of such geometric insights to diversity control remains largely underexplored. The most relevant prior work, such as Scendi~\cite{ospanov2025scendi} and SPARKE~\cite{jalalisparke}, seeks to ground T2I diversity by decomposing the generated images in CLIP space~\cite{radford2021clip} with prompt-aware and model-aware components on high-level entropy estimates on distinct text-image pairs. 
Crucially, metrics like the Scendi Score degenerate to the standard VS in fixed-prompt settings due to non-invertible covariance matrices, limiting their utility for single-prompt diversity. 
In contrast, our work leverages explicit geometrical projections to define both a robust way to measure diversity and a corresponding inference-time guidance mechanism.

% Recent literature has also looked into the latent space of the T2I generative models.

\section{Spherically Disentangled Diversity Measure}
\label{sec:measure}

We now introduce our analysis for disentangling diversity within the CLIP embedding space~\cite{radford2021clip}. Specifically, we develop a geometric analysis to decompose the variance of a generated batch into distinct prompt-dependent and prompt-independent components, which enables precise measurement of diversity sources.

\subsection{Motivation and Problem Formulation}
\label{subsec:prob_measure}

Our motivation stems from the inherent under-specification nature of T2I generation: a single text prompt rarely constrains the full semantic and stylistic content of an image. Consider a fixed prompt such as \emph{``A black colored car.''} While the prompt explicitly specifies the subject (i.e., the car), it leaves a vast subspace of attributes (e.g., viewing angles, background, etc.) unspecified. This observation suggests that diversity in a generated batch manifests in two distinct forms: variations that adhere to the prompt constraints (prompt-dependent) and variations that explore the unspecified degrees of freedom (prompt-independent).

Based on this intuition, we formalize the problem of diversity measurement as follows:
Given a text prompt $c$ and a T2I model $p_\theta$, let $\mathcal{X} = \{\mathbf{x}_i\}_{i=1}^B$ denote a batch of $B$ generated images sampled from the conditional distribution $\mathbf{x} \sim p_\theta(\mathbf{x}|c)$. Our objective is to quantify the diversity of $\mathcal{X}$ not as a single scalar, but as a disentangled tuple $(\mathcal{D}_{\text{dep}}, \mathcal{D}_{\text{ind}})$, where $\mathcal{D}_{\text{dep}}$ (prompt-dependent diversity) measures the variance in how the model interprets the explicit constraints of the prompt $c$; and $\mathcal{D}_{\text{ind}}$ measures the variance of attributes that are orthogonal to the semantic direction defined by $c$.
We seek a metric space wherein these sources of variation can be geometrically isolated and measured.

\subsection{Spherical Disentanglement and Residual Analysis}
\label{subsec:spherical_disen}
To achieve the disentanglement formulated above, we require a metric space that satisfies two critical properties: Firstly, it should structurally align visual and textual representations within a shared manifold to allow geometric grounding of image variations relative to the text; Secondly, it should ideally possess a well-behaved topology, such as a hypersphere, to enable non-divergent diversity comparison through geometric constraints.
We therefore adopt the $d$-dimensional CLIP embedding space~\cite{radford2021clip} for our analysis, where its explicit normalization restricts all embeddings to a high-dimensional unit hypersphere $\mathbb{S}^{d-1}$.

Within this spherical geometry, given the normalized text embedding $\mathbf{e}_t$ and a batch of normalized image embeddings $\mathcal{P} = \{\mathbf{e}_i\}_{i=1}^B$. 
We can strictly decompose each $\mathbf{e}_i$:
\begin{equation}
\label{eq:1}
    \mathbf{e}_i = \sum_{k=1}^d \lambda_k \mathbf{u}_k,  \: \mathbf{u}_m^\top\mathbf{u}_n = 0 \: \forall m \neq n;
    \end{equation}
where $\{\mathbf{u}_k\}_{k=1}^d$ forms an orthonormal basis of the embedding space, and $\lambda_k = \mathbf{e}_i^\top \mathbf{u}_k$ are the scalar projection coefficients.
By construction, we align the first basis vector with the text prompt (i.e., $\mathbf{u}_1 = \mathbf{e}_t$), so that the first term $\lambda_1 \mathbf{u}_1$ captures the prompt-dependent component, while the remaining terms represent the orthogonal residual. We can therefore rewrite Eq.~\ref{eq:1} in CLIP embedding space as follows:
\begin{equation}
\mathbf{e}_i = (\mathbf{e}_i^\top \mathbf{e}_t) \mathbf{e}_{t} + \sum_{k=2}^d(\mathbf{e}_i^\top \mathbf{u}_k) \mathbf{u}_{k}.    
\end{equation}
The scalar projection $\mathbf{e}_i^\top \mathbf{e}_t$ is mathematically equivalent to the CLIPScore~\cite{hessel2021clipscore} for quantifying the semantic consistency between the generated image and the prompt.
In theory, a complete characterization of the prompt-independent residual would require analyzing the distribution across the entire $(d-1)$-dimensional orthogonal subspace spanned by $\{\mathbf{u}_k\}_{k=2}^d$. However, in practice, it is known that deep learning representations typically lie on a low-dimensional manifold~\cite{narayanan2010sample,bengio2013representation}. Consequently, the prompt-independent variation is likely to be not uniformly distributed but highly concentrated along a few principal directions. We therefore simplify the problem by seeking only the \textit{dominant} residual basis vector $\mathbf{u}_{\text{ind}}$ that maximizes the captured variance in the orthogonal complement. 
\re{We discuss alternative basis construction strategies and provide additional empirical justification for our design choices in Sec.~\ref{sec:experiments}.}

\begin{algorithm}[tb]
   \caption{Dominant Residual Basis Identification
   }
   \label{alg:basis}
\begin{algorithmic}
   \STATE {\bfseries Input:} Text embedding $\mathbf{e}_t$, image batch embeddings $\mathcal{P}=\{\mathbf{e}_i\}_{i=1}^B$, number of candidate directions $N$

    \STATE Generate $N$ direction vectors $\{\mathbf{r}_k\}_{k=1}^N$ orthogonal to $\mathbf{e}_t$ via Gram-Schmidt~\cite{leon2013gram}.
    \FOR{$k = 1$ to $N$}
       \STATE $E_k \leftarrow \frac{1}{B} \sum_{i=1}^B |\mathbf{e}_i^\top \mathbf{r}_k|$ 
   \ENDFOR
   \STATE $k^* \leftarrow \arg\max_k E_k$
   % \STATE \textbf{return} $\mathbf{u}_{\text{ind}} \leftarrow \mathbf{r}_{k^*}$ \quad \COMMENT{Prompt-independent principal residual direction}
   \STATE \textbf{return} $\mathbf{u}_{\text{ind}} \leftarrow \mathbf{r}_{k^*}$
\end{algorithmic}
\end{algorithm}

% we seek to construct a 2D orthonormal basis $\mathcal{B} = \{\mathbf{u}_{\text{dep}}, \mathbf{u}_{\text{ind}}\}$ that captures the principal modes of variation. 
% The first basis vector, which corresponds to the prompt-dependent axis, is naturally defined by the text anchor itself as $\mathbf{u}_{\text{dep}} = \mathbf{e}_t$.

% The core challenge lies in identifying the second basis vector $\mathbf{u}_{\text{ind}}$ that characterizes the principal generative residual.

To identify the optimal residual basis vector $\mathbf{u}_{\text{ind}}$, we employ a randomized search strategy within the tangent space of the text anchor. 
Specifically, we first generate a candidate set of $N$ direction vectors $\{\mathbf{r}_k\}_{k=1}^N$ that lie strictly within the hyperplane orthogonal to $\mathbf{e}_t$ (i.e., $\mathbf{r}_k^\top\mathbf{e}_t = 0$), and are orthogonal to each other (i.e., $\mathbf{r}_m^\top\mathbf{r}_n = 0$ $\forall  m \neq n$) using the Gram-Schmidt orthogonalization~\cite{leon2013gram}. 
Ideally, these candidates serve as a representative basis for the high-dimensional residual space. 
Next, to capture the dominant mode of visual variation unrelated to the text, we evaluate the alignment of the batch embeddings $\mathcal{P}$ with each candidate axis. We compute the mean absolute projection magnitude for each candidate and select the one that maximizes the captured energy:
\begin{equation}
\label{eq:3}
\mathbf{u}_{\text{ind}} = \mathop{\arg\max}_{\mathbf{r} \in \{\mathbf{r}_k\}} \frac{1}{B} \sum_{i=1}^B | \mathbf{e}_i^\top \mathbf{r} |.
\end{equation}
This identified vector $\mathbf{u}_{\text{ind}}$ represents the prompt-independent axis that best explains the residual variance of the generated batch.
The detailed algorithm is described in Algo.~\ref{alg:basis}.
In practice, we empirically observe that a small candidate set size of $N=10$ is sufficient to robustly identify the principal residual axis, which balances computational efficiency with estimation accuracy.

\begin{figure*}
    \centering
    \includegraphics[width=1.0\linewidth]{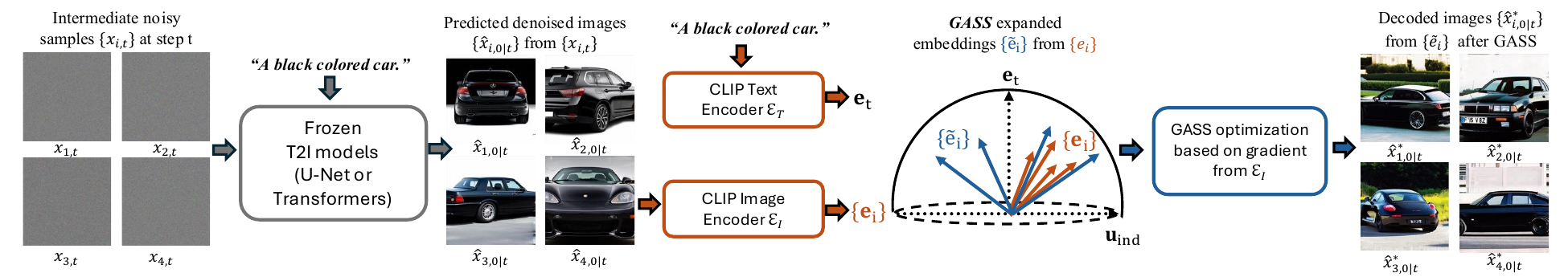}
    \caption{\textbf{Illustration of our proposed Geometry-Aware Spherical Sampling (GASS) method.} At the generation inference step $t$, original T2I sampling first estimates the predicted clean image $\hat{\mathbf{x}}_{0|t}$ based on the intermediate noisy samples $\mathbf{x}_t$, and then predict the noise we should remove from $\mathbf{x}_t$ to get $\mathbf{x}_{t-1}$. Our \emph{GASS} alters the predicted clean image from $\hat{\mathbf{x}}_{0|t}$ to $\hat{\mathbf{x}}^*_{0|t}$ through the geometric expansion (see Sec.~\ref{subsec:pertubation}) and the gradient-based optimization (see Sec.~\ref{subsec:optimization}), thus guiding the iterative sampling process with frozen generative backbones. }
    \label{fig:gass}
    % \vspace{-0.1in}
\end{figure*}

% , along with an analysis of alternative basis construction strategies explored during our design phase.

% We construct each candidate $\mathbf{r}_k$ by sampling a random vector $\mathbf{\epsilon}_k \sim \mathcal{N}(0, I)$ and projecting it onto the orthogonal complement of $\mathbf{u}_{\text{dep}}$ via Gram-Schmidt orthogonalization: $\mathbf{r}_k = \text{norm}(\mathbf{\epsilon}_k - (\mathbf{\epsilon}_k^\top \mathbf{u}_{\text{dep}})\mathbf{u}_{\text{dep}})$.
\begin{table}[t]
    \centering
    \caption{\textbf{Quantitative comparison of spherical spread scores (SPP) between generated and real images.} Real images show greater coverage spread than generated ones in both prompt-dependent ($\mathcal{D}_{\text{dep}}$) and prompt-independent ($ \mathcal{D}_{\text{ind}}$) measures. Mean and std reported over 1000 classes from ImageNet.}
    \scalebox{0.78}{
    \begin{tabular}{lcccc}
     \toprule
      Img Source   & ClipScore & $ \mathcal{D}_{\text{dep}}$ & $ \mathcal{D}_{\text{ind}}$ & $SPP$ \\ \hline
      SD2.1 & \textbf{0.303}$\pm$0.03 & 0.071$\pm$0.02 & 0.075$\pm$0.02 & 0.146$\pm$0.04\\
      SD3-M & 0.302$\pm$0.02 & 0.060$\pm$0.03 & 0.065$\pm$0.03 & 0.126$\pm$0.05 \\ 
      Real   & 0.293$\pm$0.02 & \textbf{0.110}$
      \pm$0.03 & \textbf{0.110}$\pm$0.02 & \textbf{0.220}$\pm$0.05\\ \bottomrule
    \end{tabular}}
    \label{tab:SPP_result}
    % \vspace{-0.15in}
\end{table}

\subsection{Spherical Spread Score for Diversity Measure}
\label{subsec:SSP_score}

Having established the orthogonal basis $\mathcal{B} = \{\mathbf{e}_{\text{t}}, \mathbf{u}_{\text{ind}}\}$, we now define a quantitative measure of diversity that captures the dispersion of generated images along these two axes. Specifically, we project the batch of image embeddings $\mathcal{P} = \{\mathbf{e}_i\}_{i=1}^B$ onto each basis vector and quantify the spread of the projected values, yielding two scalar diversity metrics defined as below:
\begin{equation}
\begin{aligned}
    & \mathcal{D}_{\text{dep}} = \max_{i} (\mathbf{e}_i^\top \mathbf{e}_t) - \min_{i}(\mathbf{e}_i^\top \mathbf{e}_t), \\
    & \mathcal{D}_{\text{ind}} = \max_{i} (\mathbf{e}_i^\top \mathbf{u}_{\text{ind}}) - \min_{i} (\mathbf{e}_i^\top \mathbf{u}_{\text{ind}}).
\end{aligned}
\end{equation}
$\mathbf{e}_i^\top \mathbf{u}$ stands for projection scalar of $\mathbf{e}_i$ onto the basis vector $\mathbf{u}$. 
Then we further define our overall diversity spread score as the sum of two spread scores: $SPP = \mathcal{D}_{\text{dep}} + \mathcal{D}_{\text{ind}}$.

Intuitively, these spread scores should effectively distinguish between image sets with varying diversity levels under the same text constraint. To empirically verify this, we compare the spread scores of real images from the ImageNet validation set~\cite{russakovsky2015imagenet,deng2009imagenet} against samples generated by SD2.1~\cite{rombach2022stabled} and SD3-M~\cite{esser2024sd3} using the template prompt ``A photo of \emph{[class label]}''. We observe that real-world data, which reflects natural complexity, yields significantly higher spread scores (approximately 50\% increase) compared to the generated distributions, as detailed in Tab.~\ref{tab:SPP_result}. 
\section{\emph{GASS} for Improved T2I Diversity}
\label{sec:gass_method}

Building on our geometric analysis above, we introduce \emph{GASS} (\textbf{G}eometry-\textbf{A}ware \textbf{S}pherical \textbf{S}ampling) to intervene in the generation inference process. Formally, we aim to increase the diversity score $SPP$ defined in Sec.~\ref{subsec:SSP_score}, thereby pushing the newly generated set $\mathcal{X}' = \{x'_i\}_{i=1}^B$ to cover a wider spread of the manifold measured on the CLIP sphere.

% \begin{figure*}
%     \centering
%     \includegraphics[width=1.0\linewidth]{icml2026/figs/fig2_pipeline.pdf}
%     \caption{\textbf{Illustration of our proposed Geometry-Aware Spherical Sampling (GASS) method at the generation step $t$.} At generation step $t$, original T2I sampling first estimates the predicted clean image $\hat{\mathbf{x}}_{0|t}$ based on the intermediate noisy samples $\mathbf{x}_t$, and then predict the actual noises we should remove from $\mathbf{x}_t$ to get $\mathbf{x}_{t-1}$. Our \emph{GASS} alters the predicted clean image from $\hat{\mathbf{x}}_{0|t}$ to $\hat{\mathbf{x}}^*_{0|t}$, thus guiding the iterative sampling process with frozen generative backbones. }
%     \label{fig:gass}
% \end{figure*}

% \vspace{-0.05in}
\subsection{Latent Dynamic Spherical Guidance}
\label{subsec:pertubation}

The first technical challenge is to design a latent perturbation method in the CLIP sphere such that the resulting set of image embeddings $\tilde{\mathcal{P}} = \{\tilde{\mathbf{e}}_i\}_{i=1}^B$ achieves a better spherical spread. Instead of adding isotropic noise in the high-dimensional space like most existing methods~\cite{pg2024,kirchhof2025spell,sadatcads}, our geometric framework allows us to inject and control diversity along the disentangled directions.

\begin{algorithm}[tb]
   \caption{Optimization based on CLIP Gradient 
   }
   \label{alg:opt}
\begin{algorithmic}
   \STATE {\bfseries Input:} Current batch estimates $\{\hat{x}_{i, 0|t}\}_{i=1}^B$, target embeddings $\tilde{\mathcal{P}}=\{\tilde{\mathbf{e}}_i\}_{i=1}^B$, CLIP encoder $\mathcal{E}_I$, step size $\eta$
   % \STATE {\bfseries Output:} Optimized batch $\{\hat{x}_{i, 0|t}^*\}_{i=1}^B$
   
   \STATE Encode batch estimates: $\{\mathbf{e}_i\}_{i=1}^B = \mathcal{E}_I(\{\hat{x}_{i, 0|t}\}_{i=1}^B)$
   % \STATE {\bfseries Compute Spherical Spread Loss:}
   \STATE $\mathcal{L}_{\text{spp}} = \sum_{i=1}^B (1 - \mathbf{e}_i^\top \tilde{\mathbf{e}}_i)$ \COMMENT{spherical spread loss}
   
   % \STATE {\bfseries Gradient Optimization:}
   % \STATE Calculate gradients: $\mathbf{G} = \nabla_{\{\hat{x}_{i, 0|t}\}} \mathcal{L}_{\text{spp}}$
   \STATE  $\hat{x}_{i, 0|t}^* \leftarrow \hat{x}_{i, 0|t} - \eta \cdot \nabla_{\hat{x}_{i, 0|t}} \mathcal{L}_{\text{spp}}$ \quad \text{for}\: $i=1 \dots B$
   
   \STATE \textbf{return} $\{\hat{x}_{i, 0|t}^*\}_{i=1}^B$
\end{algorithmic}
\end{algorithm}

\paragraph{Projection Expansion.}
For each image $i$ in the batch, we sample an expansion shift $\delta_{i}^k$ from a uniform distribution: 
\begin{equation}
\label{eq:5}
\delta^{k}_i \sim \mathcal{U}[-r_k, r_k],
\end{equation}
where $r_k>0$ is a hyperparameter controlling the expansion range for the specific axis (i.e., $r_{\text{dep}}$ for prompt-dependent variation along $\mathbf{e}_t$, and $r_{\text{ind}}$ for prompt-independent variation along $\mathbf{u}_{\text{ind}}$).
The perturbed target embedding $\tilde{\mathbf{e}}_i$ is then obtained by modulating the original decomposition:
\begin{equation}
\label{eq:6}
\tilde{\mathbf{e}}_i = (\mathbf{e}_i^\top \mathbf{e}_t + \delta_{i}^{\text{dep}})\mathbf{e}_{t} + (\mathbf{e}_i^\top \mathbf{u}_{\text{ind}} + \delta_{i}^{\text{ind}})\mathbf{u}_{\text{ind}} + \mathbf{r}_i,
\end{equation} 
where $\mathbf{r}_i$ represents the initial residual of $\mathbf{e}_i$ after removing the two principal components in $\mathbf{e}_t$ and $\mathbf{u}_{\text{ind}}$, defined as $\mathbf{r}_i = \mathbf{e}_i - (\mathbf{e}_i^\top \mathbf{e}_t) \mathbf{e}_{t} - (\mathbf{e}_i^\top \mathbf{u}_{\text{ind}}) \mathbf{u}_{\text{ind}}$.

\paragraph{Re-normalization.}
After obtaining the perturbed vector $\tilde{\mathbf{e}}_i$ from Eq.~\ref{eq:6}, we further apply re-normalization by projecting it back onto the unit hypersphere $\tilde{\mathbf{e}}_i \leftarrow \frac{\tilde{\mathbf{e}}_i}{|| \tilde{\mathbf{e}}_i ||_2}$. This step ensures that the guided target remains a valid representation within the CLIP embedding manifold, and empirically proven to be beneficial for the generation quality in our experiments in Sec.~\ref{sec:experiments}.

% \yz{TODO: add the theoretical justification from the hyper volume perspective.}

\paragraph{Theoretical Justifications.}
Intuitively, our latent spherical guidance directly increases the spread of the image batch in the CLIP embedding manifold. In fact, we can theoretically prove that the expected hypervolume is guaranteed to increase, as detailed below.

\begin{proposition}[Expected Geometric Volume Guarantee]
Consider a batch of $B$ points $\mathcal{P}=\{\mathbf{e}_i\}_{i=1}^B \subset \mathbb{S}^{d-1}$ on the CLIP hypersphere, where $\mathbb{S}^{d-1} \subset \mathbb{R}^d$. For each $\mathbf{e}_i$ after our proposed \emph{GASS} guidance defined in Eq.~\ref{eq:6}, the new set  $\tilde{\mathcal{P}}=\{\tilde{\mathbf{e}}_i\}_{i=1}^B$ has the expected hypervolume $\mathbb{E}[V(\tilde{\mathcal{P}})] > V(\mathcal{P})$.
\label{prop}
\end{proposition}

The key theoretical insight is that our \emph{GASS} guidance expands the Gram matrix determinant of the point set formed by the batch of images, which translates to the increased geometric hypervolume. 
A detailed proof of the Proposition~\ref{prop} is included in Appendix~\ref{app_sec:theory}.

% \begin{algorithm}[tb]
%    \caption{Optimization based on CLIP Gradient 
%    }
%    \label{alg:opt}
% \begin{algorithmic}
%    \STATE {\bfseries Input:} Current batch estimates $\{\hat{x}_{i, 0|t}\}_{i=1}^B$, Target embeddings $\{\tilde{\mathbf{e}}_i\}_{i=1}^B$, CLIP Encoder $\mathcal{E}_I$, Step size $\eta$
%    % \STATE {\bfseries Output:} Optimized batch $\{\hat{x}_{i, 0|t}^*\}_{i=1}^B$
   
%    \STATE Encode batch estimates: $\hat{\mathbf{E}} = \mathcal{E}_I(\{\hat{x}_{i, 0|t}\})$
%    \STATE {\bfseries Compute Spherical Spread Loss:}
%    \STATE $\mathcal{L}_{\text{spp}} = \sum_{i=1}^B (1 - \hat{\mathbf{e}}_i^\top \tilde{\mathbf{e}}_i)$ \COMMENT{batch-wise alignment}
   
%    \STATE {\bfseries Gradient Optimization:}
%    % \STATE Calculate gradients: $\mathbf{G} = \nabla_{\{\hat{x}_{i, 0|t}\}} \mathcal{L}_{\text{spp}}$
%    \STATE  $\hat{x}_{i, 0|t}^* \leftarrow \hat{x}_{i, 0|t} - \eta \cdot \nabla_{\hat{x}_{i, 0|t}} \mathcal{L}_{\text{spp}}$ \quad \text{for}\: $i=1 \dots B$
   
%    \STATE \textbf{return} $\{\hat{x}_{i, 0|t}^*\}_{i=1}^B$
% \end{algorithmic}
% \end{algorithm}

\subsection{SPP Gradient Optimization for T2I Generation}
\label{subsec:optimization}

After defining the target embedding set $\tilde{\mathcal{P}}=\{\tilde{\mathbf{e}}_i\}_{i=1}^B$ with enlarged spherical spread, the second technical challenge is to transfer this geometric intervention back into the generative sampling process.
Due to the fact that the CLIP does not have pre-trained decoder to directly convert the latent interventions to pixel space, we propose to translate the latent expansion by leveraging the gradients from the frozen image encoder for guiding the generation. Specifically, at a sampling step $t$, we first estimate the denoised image $\hat{x}_{0|t}$ from the current noisy latent $\mathbf{x}_t$ using the base T2I model’s prediction. This estimate is then fed into the CLIP image encoder $\mathcal{E}_I$ to obtain its current batch embedding  $\mathbf{e} = \mathcal{E}_I(\hat{x}_{ 0|t})$.
To align the generation with our diversity target $\tilde{\mathbf{e}}$ after \emph{GASS} guidance, we define a batch-wise loss $\mathcal{L}_{\text{SPP}}$ that measures the alignment between the current estimated embedding and the updated target after geometric expansion:
\begin{equation}
\mathcal{L}_{\text{SPP}} = \sum_{i=1}^B \left( 1 - \mathcal{E}_I(\hat{x}_{i, 0|t})^\top \tilde{\mathbf{e}}_i \right).
\end{equation}
Crucially, instead of modifying the noise prediction $\epsilon_\theta$, which would require backpropagation through the generative backbone, we directly optimize the estimated clean image $\{\hat{x}_{i, 0|t}\}_{i=1}^B$.
We compute the gradient of the loss and apply a correction step for each sample in the batch:
\begin{equation}
\hat{x}_{i, 0|t}^* \leftarrow \hat{x}_{i, 0|t} - \eta \cdot \nabla_{\hat{x}_{i, 0|t}} \mathcal{L}_{\text{SPP}},
\end{equation}
where $\eta$ is the learning rate. This optimized estimate is then substituted into the transition step of existing solvers from pre-trained T2I models, thus effectively steering the generation towards the diverse targets.
Detailed optimization algorithm is in Algo.~\ref{alg:opt}.
The overall pipeline of \emph{GASS} at a generative sampling step $t$ is illustrated in Fig.~\ref{fig:gass}, with the complete algorithm in Appendix~\ref{appsec:spherical_design}.

% \begin{algorithm}[tb]
%   \caption{GASS Optimization}
%   \label{alg:opt}
%   \begin{algorithmic}
%     \STATE {\bfseries Input:} data $x_i$, size $m$
%     \REPEAT
%     \STATE Initialize $noChange = true$.
%     \FOR{$i=1$ {\bfseries to} $m-1$}
%     \IF{$x_i > x_{i+1}$}
%     \STATE Swap $x_i$ and $x_{i+1}$
%     \STATE $noChange = false$
%     \ENDIF
%     \ENDFOR
%     \UNTIL{$noChange$ is $true$}
%   \end{algorithmic}
% \end{algorithm}

\begin{figure*}[th]
    \centering
    \includegraphics[width=1.0\linewidth]{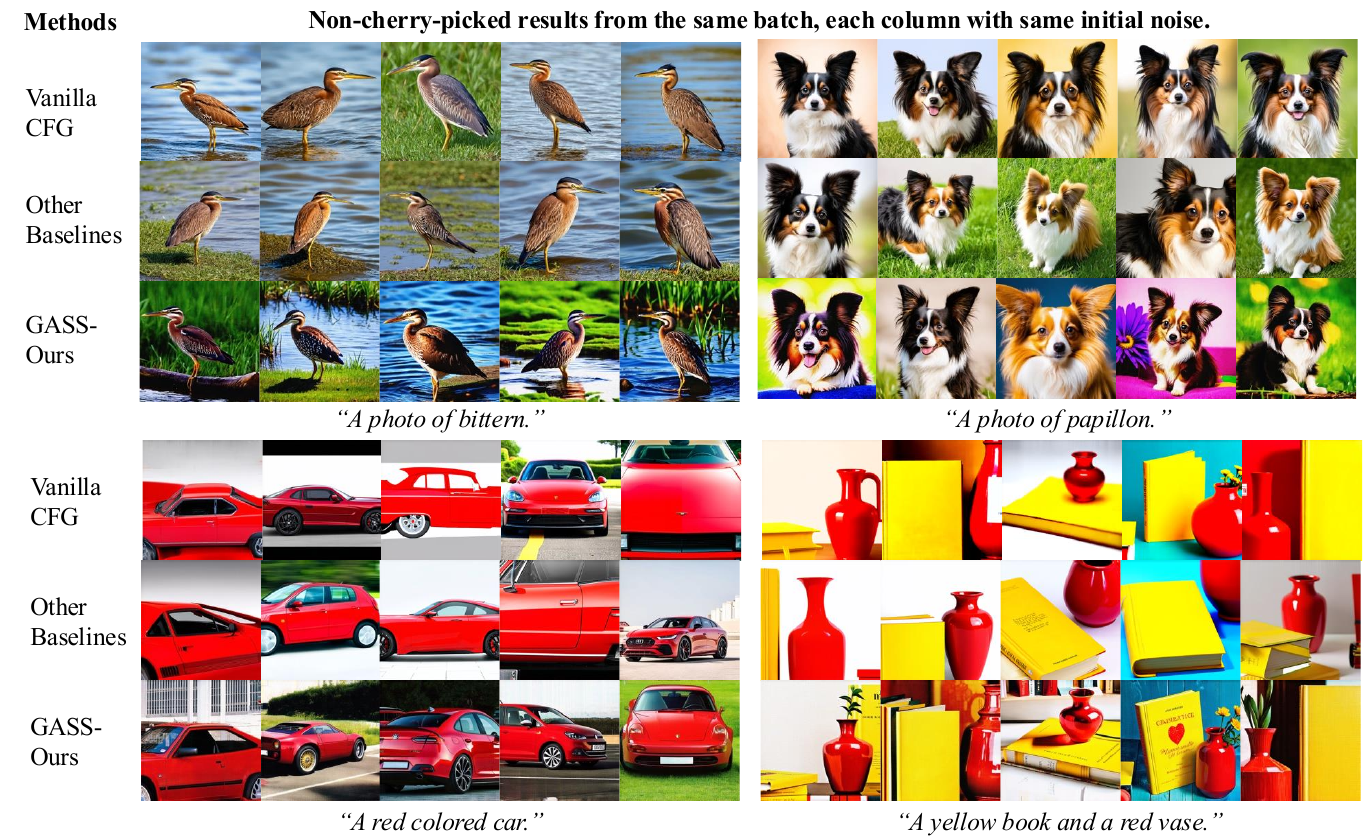}
    \caption{\textbf{Non-cherry-picked qualitative comparisons with other diversity enhancement methods on ImageNet~\cite{russakovsky2015imagenet} and Drawbench~\cite{saharia2022photorealistic}.} Compared to other methods (i.e., PG~\cite{pg2024}, CADS~\cite{sadatcads}, IG~\cite{kynkaanniemi2024ig}, and SPELL~\cite{kirchhof2025spell}), our proposed \emph{GASS} generates images with both richer semantic variation (e.g., object poses and layout) and more detailed and diverse backgrounds.}
    \label{fig:qualitative}
    \vspace{-0.1in}
\end{figure*}

\begin{table*}[h]
    \centering
    \caption{\textbf{Quantitative evaluations on ImageNet with SD2.1 and SD3-M as base T2I models.}}
    \scalebox{0.84}{
    \begin{tabular}{l|ccccc|ccccc}
    \toprule
       Base Models  &  \multicolumn{5}{c|}{SD2.1} & \multicolumn{5}{c}{SD3-M} \\ \toprule    
       Methods & Density$\uparrow$ & Coverage$\uparrow$ & VS$\uparrow$ & ClipScore$\uparrow$ & SPP$\uparrow$ &Density$\uparrow$ & Coverage$\uparrow$ & VS$\uparrow$ & ClipScore$\uparrow$ & SPP$\uparrow$ \\ \hline
        % Vanilla\\
        Vanilla Guidance (CFG) & \textbf{1.130} & \textbf{0.623} & 31.826 & 0.308 & 0.143 & 1.105 & 0.588 & 28.119 & 0.308 & 0.137 \\ 
        PG-ICLR'24 & 1.039 & 0.614 &32.082 & \textbf{0.309} &0.146 & 1.103& 0.586 & 28.119 & 0.308 & 0.129 \\
        CADS-ICLR'24 & 0.897 & 0.567 & 28.935 & 0.306& 0.142 & 1.374 & \textbf{0.636} & 28.456 & 0.309 & 0.133 \\
        IG-NeurIPS'24 & 0.901 &  0.574 &  28.935 &  0.306 & 0.140 & \textbf{1.389} & 0.627 & 27.415 & 0.310 &0.129 \\
        SPELL-ICML'25 & 0.912& 0.578 & 32.601 & 0.305 & 0.144 & 1.105 & 0.585 & 28.433&0.302 & 0.128 \\
        \emph{GASS} (ours) & 1.015 & 0.603 & \textbf{32.711} & \textbf{0.309} & \textbf{0.149} & 1.164 & 0.611& \textbf{28.877}&  \textbf{0.313} & \textbf{0.141}
         \\ \bottomrule
    \end{tabular}}
    \label{tab:imagenet}
    % \vspace{-0.12in}
\end{table*}

\vspace{-0.05in}
\section{Experiments}
\label{sec:experiments}

In this section, we describe our experimental setup and present our results and ablation studies~\footnote{All experiments were conducted by LIX and Princeton.}.

\subsection{Experimental Setup}
\label{subsec:exp_setup}

% Backbone architecture, and their brief description of training and sampling strategies.
% SD2.1 has been deleted due to the EU AI ACT
\paragraph{Base T2I Models.}
To demonstrate the general applicability of \emph{GASS}, we employ Stable Diffusion 2.1 (SD2.1) and SD3 Medium (SD3-M)~\cite{esser2024sd3} as our frozen generative backbones. These choices cover a wide spectrum of modern text-to-image generative models, spanning diffusion~\cite{ho2020denoising} versus rectified flow~\cite{liu2023rectifiedflow} generation paradigms, and U-Net~\cite{ronneberger2015unet} versus DiT~\cite{peebles2023dit} architectures.

\paragraph{Dataset and Benchmarks.}
We evaluate our method on ImageNet-1K~\cite{deng2009imagenet,russakovsky2015imagenet} and DrawBench~\cite{saharia2022photorealistic}. For ImageNet, we synthesize 50 images per class using the standard template ``A photo of [\emph{class label}]’’. For DrawBench, we generate 10 samples per prompt and batch. While ImageNet serves as a standard baseline as in prior literature~\cite{kirchhof2025spell}, DrawBench features prompts with higher semantic complexity and structural constraints, providing a more rigorous testbed for fine-grained diversity analysis.

\paragraph{Metrics.}
Our evaluation covers sample diversity, generative quality, and semantic consistency alignment. For ImageNet, we employ the classic Density and Coverage~\cite{naeem2020reliable} as indicators of fidelity and diversity, respectively. We complement these with ClipScore~\cite{hessel2021clipscore} for alignment, and VS~\cite{friedmanvendi2023} for intrinsic diversity.
For DrawBench, due to the absence of reference images, we utilize reference-free metrics: ImageReward~\cite{xu2023imagereward} for perceptual quality, VS for diversity, and ClipScore for consistency.
Additionally, we also report our proposed SPP to quantify the geometric spread of the generated samples.

\begin{table*}[h]
\centering
\caption{\textbf{Quantitative evaluations on DrawBench with SD2.1 and SD3-M as base T2I models.} }
\scalebox{0.95}{
\begin{tabular}{l|cccc|cccc}
\toprule
Base Models & \multicolumn{4}{c|}{SD2.1} & \multicolumn{4}{c}{SD3-M} \\ \toprule
Methods & VS$\uparrow$ & ImageReward$\uparrow$ & ClipScore$\uparrow$ & SPP$\uparrow$ & VS$\uparrow$ & ImageReward$\uparrow$ & ClipScore$\uparrow$ & SPP $\uparrow$ \\
\midrule
% Vanilla & 9.723 &  -1.243 & 0.256 & 0.167 & 9.518 &  -0.521  & 0.281 & 0.166 \\
Vanilla Guidance (CFG) & 8.599 & 0.217 & 0.306 & 0.122 & 8.115 &  0.779 & 0.318 & 0.113 \\
PG-ICLR'24 & 8.637 &  \textbf{0.231} & 0.306 &0.121 & 8.002 & \textbf{0.799} & 0.318 & 0.113\\
CADS-ICLR'24 & 8.784 & 0.173 & 0.305 & 0.125 & 8.118 &  0.726 &  0.316 & 0.112  \\
IG-NeurIPS'24 & 8.761 & 0.215  & 0.306 & 0.125 & 8.002 & 0.798& 0.318& 0.112 \\
SPELL-ICML'25 & 8.726 & 0.205 &  0.305 &  0.131& 8.166 & 0.671 & 0.315 & 0.112\\
\emph{GASS} (ours) & \textbf{8.847}  & 0.229 & \textbf{0.307} & \textbf{0.135} & \textbf{8.212} & 0.778 & \textbf{0.320}& \textbf{0.114} \\
\bottomrule
\end{tabular}}
\label{tab:drawbench}
% \vspace{-0.15in}
\end{table*}

\begin{table}[h]
    \centering
    \caption{\textbf{Quantitative analysis w.r.t. different levels of prompt complexities from Drawbench.} We show the performance change compared to the CFG baseline. For instance, 8.301 $\rightarrow$ 8.535 means that the metric changes from 8.301 (CFG baseline) to 8.535 (with GASS).}
    \label{tab:prompt_complexity}
    \scalebox{0.73}{
    \begin{tabular}{lcccc}
    \toprule
      Prompt   & VS$\uparrow$ & ImageReward$\uparrow$ & ClipScore$\uparrow$ & SPP$\uparrow$ \\  \hline
      
       Short   & 8.301$\rightarrow$\textbf{8.535} & \textbf{0.748}$\rightarrow$0.698 & \textbf{0.322}$\rightarrow$0.321 & 0.113$\rightarrow$\textbf{0.121} \\
       Medium & 7.663$\rightarrow$\textbf{7.918} & \textbf{1.115}$\rightarrow$1.043 & 0.318$\rightarrow$\textbf{0.322} & 0.103$\rightarrow$\textbf{0.107} \\
       Long & 7.549$\rightarrow$\textbf{7.935} & 0.572$\rightarrow$\textbf{0.622} & 0.310$\rightarrow$0.310 & 0.092$\rightarrow$\textbf{0.099}\\
    \bottomrule
    \end{tabular}}
\end{table}

\paragraph{Diversity Enhancement Baselines.}
We compare our approach with four recent and state-of-the-art (SOTA) sampling-based methods designed to enhance sample diversity under fixed prompts: Particle Guidance (PG)~\cite{pg2024}, CADS~\cite{sadatcads}, IG~\cite{kynkaanniemi2024ig}, and SPELL~\cite{kirchhof2025spell}. All of these methods are inference-time interventions that amplify diversity by introducing stochastic perturbations to either the intermediate latents or the conditioning signals during the generation sampling trajectory.
We replicate all baseline methods, either using their official public implementations or re-implementing them based on the original papers. 
% We compare our performance with four recent methods that focus on diversifying T2I models in the same setting, which aim to improve a set of generated images with a fixed given prompt. These methods include Particle Guidance (PG)~\cite{pg2024}, CADS~\cite{sadatcads}, IG~\cite{kynkaanniemi2024ig}, and SPELL~\cite{kirchhof2025spell}. All of the baseline methods are sampling-based techniques that introduce perturbations in the generation trajectories, either in the intermediate noisy latents or in the conditions.
% In addition, we also include the vanilla base model w/o CFG as a vanilla baseline.
% Specifically, Particle Guidance (PG)~\cite{pg2024} introduces a joint particle time-evolving potential guidance that alters the sampling process to be non independent to enhance the diversity among those samples.

\paragraph{Implementation Details.}
To test the generalization ability across varying scales, we generate images at $768^2$ for ImageNet and $512^2$ for Drawbench. 
During the \emph{GASS} gradient optimization stage, we employ the Adam optimizer with a learning rate of $1 \times 10^{-4}$ for a maximum of 60 steps. We utilize an early stopping strategy with a tolerance of $5 \times 10^{-4}$ and patience of 4 optimization steps. 
The default inference steps are set to be 50 and 28 for SD2.1 and SD3.5-M, respectively.
The default expansion ranges are set to be $r_{\text{dep}} = r_{\text{ind}} = 0.02$.
Our proposed \emph{GASS} is a sparse guidance mechanism that can be activated only over a specified interval of sampling steps, reducing computational overhead.
For \re{10$\sim$20} \emph{GASS} guidance steps along the SD3-M generation trajectory, the average cost to sample a batch of images is around \re{2.93$\sim$3.68 seconds}  on a Nvidia A100 GPU.
% All experiments are conducted on an AWS instance equipped with 8 Nvidia A100 GPU. 
% For implementation details 
% Further details are provided in the Appendix.
% \yz{Need to add the computational overhead.}

% \begin{figure}[h]
%     \centering
%     \includegraphics[width=1.0\linewidth]{icml2026/figs/fig4_control.pdf}
%     \caption{\textbf{\emph{GASS} controls on the source of diversity by expanding the geometric spread along specified direction.} Specially, \emph{GASS} on prompt-dependent axis $\mathbf{e}_t$ diversifies images through variations via poses and layout, while expansion along prompt-independent direction $\mathbf{u}_{\text{ind}}$ changes attributes like background and styles.}
%     \label{fig4:control}
% \end{figure}

\begin{figure}[t]
    \centering
    \includegraphics[width=1.0\linewidth]{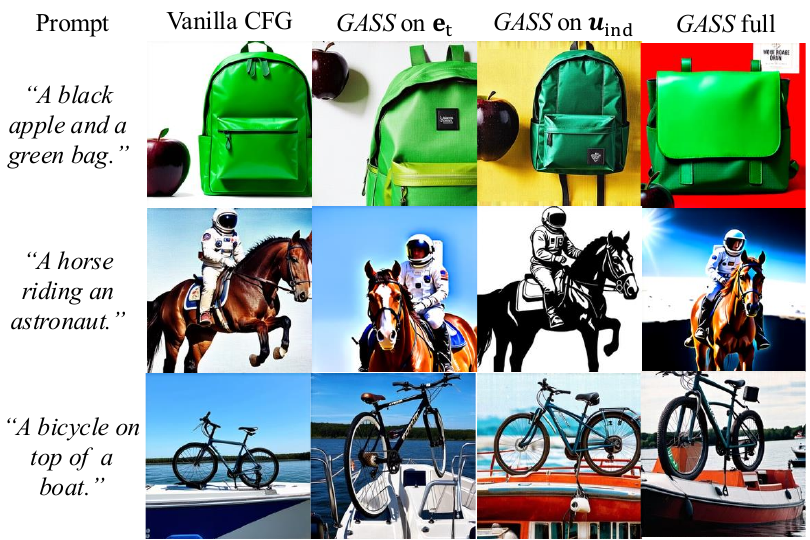}
    \caption{\textbf{\emph{GASS} controls the source of diversity by expanding the geometric spread along specified directions.} Specially, \emph{GASS} on prompt-dependent axis $\mathbf{e}_t$ diversifies images through variations via poses and layout, while expansion along prompt-independent direction $\mathbf{u}_{\text{ind}}$ changes attributes like background and styles.}
    \label{fig4:control}
    % \vspace{-0.15in}
\end{figure}

\subsection{Main Results and Analysis}
\label{subsec:results}

\paragraph{Diversity Comparison.}
We report the main results on ImageNet and DrawBench in Tab.~\ref{tab:imagenet} and Tab.~\ref{tab:drawbench}, respectively, where we evaluate the generated images in terms of diversity, perceptual quality, and text–image consistency alignment. Compared to recent diversity-enhancement methods that primarily maximize intra-batch sample dissimilarity, our approach improves diversity in both reference-based and reference-free evaluations, while maintaining competitive quality and consistency. Notably, \emph{GASS} achieves the largest gains on diversity-oriented metrics (e.g., VS~\cite{friedmanvendi2023}) with minimal degradation, or even slight improvements, on quality and alignment metrics, highlighting the effectiveness of our geometry-aware design.
This is further qualitatively demonstrated in Fig.~\ref{fig:qualitative}, where we present \textbf{non-cherry-picked} comparisons with baseline methods across both benchmarks. Notably, \emph{GASS} not only introduces semantic variations comparable to other methods but also generates \textbf{significantly more detailed backgrounds}. In contrast, other methods often produce ambiguous and smoothed background regions. We attribute this improvement to our explicit expansion of the geometric spread along the prompt-independent orthogonal direction.

\paragraph{Controllability in Disentangled Diversity Sources.}
Given our orthogonal basis decomposition, we can selectively control diversity enhancement from specific sources by modulating the expansion range $r_k$. In Fig.~\ref{fig4:control}, we illustrate this disentanglement controllability by expanding along the prompt-dependent direction ($\mathbf{e}_t$), prompt-independent direction ($\mathbf{u}_{\text{ind}}$), or both. We empirically observe that prompt-dependent expansion introduces semantic variations such as layout and object poses, while prompt-independent expansion generates diversity through backgrounds and styles.

\paragraph{Correlation among Diversity, Quality and Alignment.}
Prior work~\cite{zhang2025intricate,astolfi2024consistency} explores evaluation perspectives including diversity, quality, and alignment. Consistent with their findings, Tab.~\ref{tab:imagenet} and Tab.~\ref{tab:drawbench} reveal this similar trade-off, where diversity gains typically incur quality drops across different methods. In general, our approach achieves superior diversity improvements with minimal quality degradation.

\paragraph{Diversity under More Complex Prompts.}
In addition, it is worth noting that previous works~\cite{ospanov2025scendi,zhang2025intricate} reveal that image diversity can often be introduced through more detailed and specified prompts, thus obscuring the true effect of diversity  enhancement from the model perspective. Notably, we demonstrate that \emph{GASS} introduces diversity even under these conditions,  as shown in Fig.~\ref{fig:prompt}. While vanilla CFG already outputs more diverse images with more specified prompts, our method still  introduces further variations across unspecified attributes.

\re{We further conduct a quantitative analysis.
Specifically, we partition the 200 text prompts from the DrawBench~\cite{saharia2022photorealistic} into three distinct categories based on word count: short ( $\le$ 8 words), medium ( $9-14$ words), and long/complex ( $\ge$ 15 words). These categories consist of 92, 62, and 46 prompts, respectively. The qualitative breakdown is provided below in Tab.~\ref{tab:prompt_complexity}, where we indicate the performance gain compared to the CFG baseline. }
\re{Interestingly, while human-perceived visual diversity appears to increase with longer and more complex prompts, we observe that diversity metrics (e.g., VS and SPP) actually exhibit a decreasing trend. Despite this, our proposed \textit{GASS} consistently enhances diversity across all prompt complexity categories. Notably, based on the Vendi scores, the margin of improvement achieved by GASS becomes increasingly pronounced as prompt length and complexity grow.}

% For instance, 8.301 $\rightarrow$ 8.535 means that the metric increases from 8.301 (CFG baseline) to 8.535 (with GASS).}

% \re{We further conduct quantitative studies w.r.t different levels of complexities, with the detailed results in Sec.~\ref{appsubsec:ablation}}.

% \vspace{-0.05in}
\subsection{Ablation Studies and Analysis}
\label{subsec:ablation}

We further investigate several key designs of the proposed \emph{GASS} method through ablation studies.
Tab. ~\ref{tab:ablation_redisual} and Tab.~\ref{tab:ablation} summarize the quantitative evaluation results.

\paragraph{Alternative Decomposition Basis Construction Methods.}
\re{
To demonstrate the effectiveness of the basis selection method presented in Sec.~\ref {subsec:spherical_disen}, we further investigate two alternative strategies as described below.
\textbf{Isotropic Perturbation (IP)}: We sample both perturbation directions randomly from an isotropic orthogonal basis;
\textbf{Random Direction (RD):} We keep the prompt-dependent base $\mathbf{e}_t$ unchanged, and replace $\mathbf{u}_{ind}$ with a random vector orthogonal to $\mathbf{e}_t$. }
\re{
The results are presented in Tab.~\ref{tab:ablation_redisual}, where we observe that both isotropic perturbations and a random selection of the dominant residual base on the sphere do not perform well in the T2I diversity enhancement task, demonstrating the effectiveness of our basis construction method. In particular, in the case of an isotropic perturbation, where $\mathbf{e}_t$ is replaced by a random direction, the ClipScore drops from 0.320 to 0.308, indicating that it hinders the semantic alignment between the prompt and the generated images.
}

\paragraph{Re-Normalization.}
In our proposed \textit{GASS} described in Sec.~\ref{sec:gass_method}, we re-normalize the image embedding $\tilde{\mathbf{e}_i}$ to constrain it within the unit hypersphere. Intuitively, this keeps perturbed vectors in the high-density in-distribution region. Our empirical ablation results from Tab.~\ref{tab:ablation} show that removing re-normalization negatively affects the image quality, inducing a drop of ImageReward~\cite{xu2023imagereward} and ClipScore~\cite{hessel2021clipscore} while increasing slightly on diversity measures.

\begin{table}[t]
    \centering
    \caption{\textbf{Ablation results on different strategies for constructing the decomposition basis.} \emph{IP} stands for Isotropic Perturbation, and \emph{RD} represents Random Direction.}
    \label{tab:ablation_redisual}
    \scalebox{0.94}{
    \begin{tabular}{lcccc}
    \toprule
        Method & VS$\uparrow$ & ImageReward$\uparrow$ & ClipScore$\uparrow$ & SPP$\uparrow$ \\  \hline
        IP & 8.203 & 0.774 & 0.308 & 0.113 \\
        RD & 8.206 & \textbf{0.778} & 0.313 & 0.113  \\
        \textit{GASS} & \textbf{8.212} & \textbf{0.778} & \textbf{0.320} & \textbf{0.114} \\ 
         \bottomrule
    \end{tabular}}
\end{table}

\begin{table}[t]
    \centering
    \caption{\textbf{Additional ablation results of \emph{GASS} variants.} We show the impact of the re-normalization technique, expansion range, and perturbation steps on Drawbench.}
    \scalebox{0.79}{
    \begin{tabular}{lcccc}
    \toprule
     \emph{GASS} Variants    &  VS$\uparrow$ & ImageReward$\uparrow$ & ClipScore$\uparrow$ & SPP$\uparrow$ \\ \hline
    % IP & 8.203 & 0.774 & 0.308 & 0.113 \\
    % RD & 8.206 & \textbf{0.778} & 0.313 & 0.113  \\
    % GASS & \textbf{8.212} & \textbf{0.778} & \textbf{0.320} & \textbf{0.114}
    %  \\ \hline \hline
     w/o Norm. & 8.876 & 0.732 & 0.313 & 0.123 \\ \hline
     $r_{\text{dep}}=0$, $r_{\text{ind}}=0.02$ & 8.207 & 0.787 & 0.319 & 0.111 \\
      $r_{\text{dep}}=0.02$, $r_{\text{ind}}=0$ & 8.206 & 0.780 & 0.320 & 0.112  \\
     $r_{\text{dep}}=r_{\text{ind}}=0.02$ & 8.212 & 0.778 & 0.320 & 0.114 \\
     $r_{\text{dep}}=r_{\text{ind}}=0.05$ & 8.205 & 0.778 & 0.320 & 0.112 \\ \hline
     t=10 (consecutive) & 8.215 & 0.808 & 0.318 & 0.114 \\
     t=10 (uniform) & 8.127 &  0.757 & 0.321 & 0.112 \\
     t=15 (uniform) & 8.202 & 0.784 & 0.319 & 0.113 \\
     t=20 (uniform) & 8.212 & 0.778 &  0.320 & 0.114\\
    \bottomrule
    \end{tabular}}
    \label{tab:ablation}
\end{table}

\paragraph{Expansion Range.}
In our proposed sampling method, we define expansion ranges via hyperparameters $r_{\text{dep}}$ and $r_{\text{ind}}$ along $\mathbf{e}_t$ and $\mathbf{u}_{\text{ind}}$, respectively. Ablation experiments in Tab.~\ref{tab:ablation} show that $r=0.02$ achieves optimal trade-offs across different evaluation metrics. In addition, expanding along both axes yields the best overall diversity gains compared to single direction expansion.

\textbf{Perturbation Steps.}
Our \emph{GASS} method is sparse, which requires application only over a subset of sampling steps.
\re{We evaluate two scheduler designs, uniform and consecutive \emph{GASS} steps, with results reported in Tab.~\ref{tab:ablation}. The early consecutive scheduler applies \emph{GASS} for 10 sampling steps starting from the initial noisy Gaussian; the uniform scheduler steers generation at uniformly sampled steps. The early consecutive scheduler achieves higher VS and ImageReward scores, at the cost of a marginal drop in CLIPScore. Qualitatively, consecutive \emph{GASS} steps in the early denoising trajectory tend to generate images with lower color saturation in the final generated images compared to the uniform counterpart, as illustrated in Fig.~\ref{fig:saturation}.
}

% We ablate the number of steps for the \emph{GASS} sampling, with results shown in Tab.~\ref{tab:ablation}. 
% \re{We}

% In practice, we note that 15-20 steps are sufficient to achieve good diversity improvements.

\begin{figure}[t]
    \centering
    \includegraphics[width=1.0\linewidth]{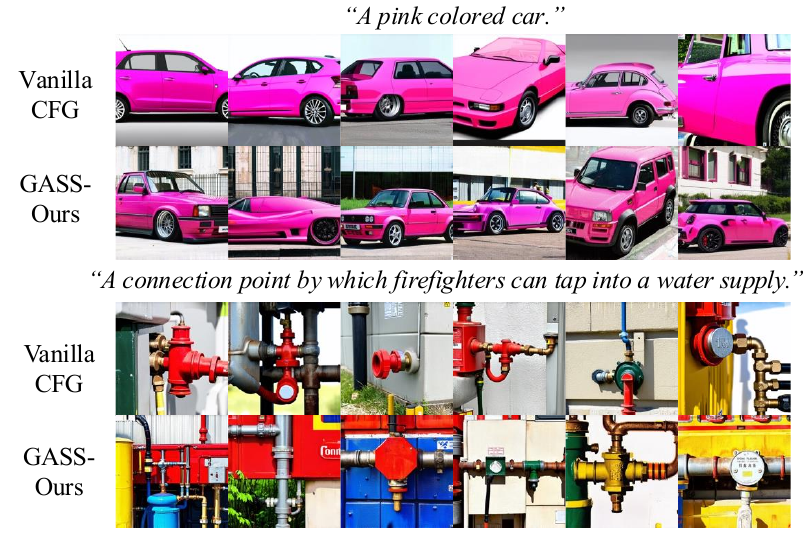}
    \caption{\textbf{\emph{GASS} still introduces generated image diversity, even when provided with more complex text prompts.}}
    \label{fig:prompt}

\end{figure}

\begin{figure}[t]
    \centering
    \includegraphics[width=1.0\linewidth]{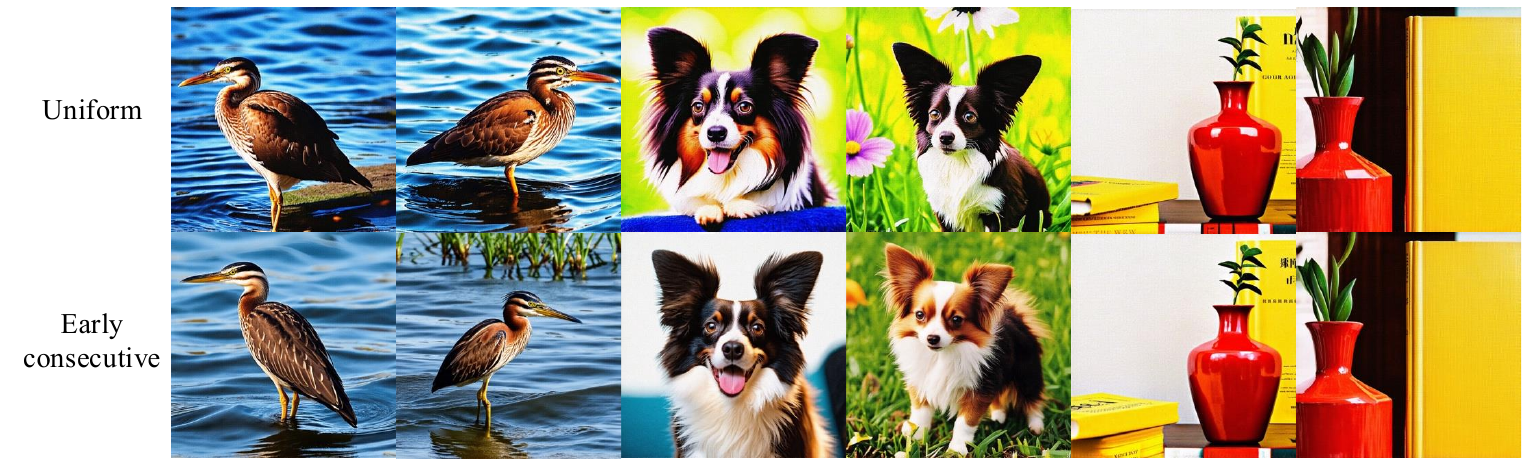}
    \caption{\textbf{\emph{GASS} can be applied sparsely on early consecutive or uniform perturbation steps along the T2I generation trajectories.} The early consecutive scheduler tends to generate images with lower color saturation compared to the uniform counterpart.}
    \label{fig:saturation}
\end{figure}

% \vspace{-0.05in}
\section{Conclusion and Discussions}
\label{sec:conclu}
% \vspace{-0.1in}

In this work, we investigate sample diversity in T2I generation through the lens of spherical geometry. By decomposing diversity into prompt-dependent and prompt-independent components grounded in the geometric structure of the CLIP space, we introduce a principled framework to quantify variation along these orthogonal directions. We further propose \emph{GASS}, a geometry-aware sampling guidance that enhances diversity in a controllable manner via dynamic interventions during inference. 
Experiments across diverse T2I backbones and benchmarks demonstrate the effectiveness and generalizability of our approach.

A potential future direction could be extending the proposed geometric decomposition beyond prompts (e.g., to multi-condition inputs such as layout or reference images) may enable finer control over which factors of variation are amplified.
Limitations and future directions are further discussed in the Appendix~\ref{appsec:discussions}.
 % We also include an extended discussion of the limitations in the Appendix~\ref{appsec:discussions}. 

% \textbf{Future Directions.}
% Several directions remain open. First, extending the proposed geometric decomposition beyond prompts (e.g., to multi-condition inputs such as layout or reference images) may enable finer control over which factors of variation are amplified. It could also be worth exploring a more direct connection between geometric spread and downstream utility, through aspects such as incorporating additional perceptual constraints or human-preferred diversity calibrations.

% \section*{Accessibility}

% Authors are kindly asked to make their submissions as accessible as possible
% for everyone including people with disabilities and sensory or neurological
% differences. Tips of how to achieve this and what to pay attention to will be
% provided on the conference website \url{http://icml.cc/}.

% \section*{Software and Data}

% If a paper is accepted, we strongly encourage the publication of software and
% data with the camera-ready version of the paper whenever appropriate. This can
% be done by including a URL in the camera-ready copy. However, \textbf{do not}
% include URLs that reveal your institution or identity in your submission for
% review. Instead, provide an anonymous URL or upload the material as
% ``Supplementary Material'' into the OpenReview reviewing system. Note that
% reviewers are not required to look at this material when writing their review.

% Acknowledgements should only appear in the accepted version.
\newpage
\section*{Acknowledgements}
This work is supported through the research grant from Meta Inc., under grant number NOA AWD1008796. 
JL also acknowledges the funding support by the French National Research Agency (ANR) via the “GraspGNNs” JCJC grant (ANR-24-CE23-3888).

% \yz{Check with other collaborators.}

% \textbf{Do not} include acknowledgements in the initial version of the paper
% submitted for blind review.

% If a paper is accepted, the final camera-ready version can (and usually should)
% include acknowledgements.  Such acknowledgements should be placed at the end of
% the section, in an unnumbered section that does not count towards the paper
% page limit. Typically, this will include thanks to reviewers who gave useful
% comments, to colleagues who contributed to the ideas, and to funding agencies
% and corporate sponsors that provided financial support.

% \newpage
\section*{Impact Statement}

This paper presents work whose goal is to advance the field of generative models by improving sample diversity in text-to-image generation, with the broader goal of mitigating potential negative societal impacts related to bias and fairness in generated images. There are potential societal consequences of our work, as is similar to the case for many other works in AI generation. However, from a technical standpoint, we do not identify any specific, unusual risks beyond those already associated with contemporary text-to-image generative models.

% , none
% which we feel must be specifically highlighted here

% Authors are \textbf{required} to include a statement of the potential broader
% impact of their work, including its ethical aspects and future societal
% consequences. This statement should be in an unnumbered section at the end of
% the paper (co-located with Acknowledgements -- the two may appear in either
% order, but both must be before References), and does not count toward the paper
% page limit. In many cases, where the ethical impacts and expected societal
% implications are those that are well established when advancing the field of
% Machine Learning, substantial discussion is not required, and a simple
% statement such as the following will suffice:

% ``This paper presents work whose goal is to advance the field of Machine
% Learning. There are many potential societal consequences of our work, none
% which we feel must be specifically highlighted here.''

% The above statement can be used verbatim in such cases, but we encourage
% authors to think about whether there is content which does warrant further
% discussion, as this statement will be apparent if the paper is later flagged
% for ethics review.

% In the unusual situation where you want a paper to appear in the
% references without citing it in the main text, use \nocite
% \nocite{langley00}

% \newpage
\bibliography{example_paper}
\bibliographystyle{icml2026}

%%%%%%%%%%%%%%%%%%%%%%%%%%%%%%%%%%%%%%%%%%%%%%%%%%%%%%%%%%%%%%%%%%%%%%%%%%%%%%%
%%%%%%%%%%%%%%%%%%%%%%%%%%%%%%%%%%%%%%%%%%%%%%%%%%%%%%%%%%%%%%%%%%%%%%%%%%%%%%%
% APPENDIX
%%%%%%%%%%%%%%%%%%%%%%%%%%%%%%%%%%%%%%%%%%%%%%%%%%%%%%%%%%%%%%%%%%%%%%%%%%%%%%%
%%%%%%%%%%%%%%%%%%%%%%%%%%%%%%%%%%%%%%%%%%%%%%%%%%%%%%%%%%%%%%%%%%%%%%%%%%%%%%%
\newpage
\appendix
\onecolumn

\section*{Appendices}

The appendix is structured as follows: 
First, Sec. ~\ref{app_sec:theory} provides the formal theoretical proof of Proposition ~\ref{prop} presented in the main paper.
Next, Sec. ~\ref{appsec:spherical_design} details the overall algorithm, summarizing the complete process of our proposed \emph{GASS} method.
Sec.~\ref{appsec:exp_results} includes additional experimental details, including additional implementation details, and more qualitative results.
Sec.~\ref{appsec:discussions} discusses the limitations and analyzes failure cases observed in our experiments, proposing several promising future directions.
% Additionally, we provide the anonymous core code of our \emph{GASS} method in the supplementary materials.

\section{Theoretical Justification on the Diversity Spread and Hypervolume Expansion after \emph{GASS}}
\label{app_sec:theory}

We provide a theoretical justification for the volume expansion induced by our proposed \emph{GASS}, as stated in Proposition~\ref{prop}.

% \begin{proposition}
\textbf{Proposition 4.1} (Expected Geometric Volume Guarantee).
\textit{
Consider a batch of $B$ points $\mathcal{P}=\{\mathbf{e}_i\}_{i=1}^B \subset \mathbb{S}^{d-1}$ on the CLIP hypersphere, where $\mathbb{S}^{d-1} \subset \mathbb{R}^d$. For each $\mathbf{e}_i$ after our proposed \emph{GASS} guidance defined in Eq.~\ref{eq:6}, the new set  $\tilde{\mathcal{P}}=\{\tilde{\mathbf{e}}_i\}_{i=1}^B$ has the expected hypervolume $\mathbb{E}[V(\tilde{\mathcal{P}})] > V(\mathcal{P}).$}

\begin{proof}

The proof proceeds by establishing an explicit relationship between the geometric hypervolume of a point set and the Gram determinant of their inner product matrix. The crucial observation is that independent and orthogonal perturbations through expansion parameter $r_k$ induce positive-definite corrections to the Gram determinant, which increases the hypervolume of the original point set.

\textbf{Step 1: Hypervolume via Gram Determinant}

We characterize the hypervolume via the Gram determinant of edge vectors under Gram-Schmidt orthogonalization for the point set $\mathcal{P}$.
Specifically, for each pairs of point $\{\mathbf{e}_i,\mathbf{e}_j\}$, we form the edge vector $\textbf{e}_{ij} = \mathbf{e}_j - \mathbf{e}_i$.
We then construct $d$ orthogonal basis as described in Sec.~\ref{subsec:spherical_disen}, including our main expansion axis $\mathbf{e}_t$ and $\mathbf{u}_{\text{ind}}$. By projecting the edge vector $\mathbf{e}_{ij}$ onto $d$ orthogonal basis, and stacking those $B-1$ linearly independent edge vectors row-wise, we can define the reduced coordinate edge matrix $\mathbf{A} \in \mathbb{R}^{d\times(B-1)}$, the $(B-1)$-dimensional hypervolume is thus given by:

\begin{equation}
    V(\mathcal{P}) = \frac{\sqrt{\text{det}(\mathbf{A}^{\top}\mathbf{A}})}{(B-1)!} = \frac{\sqrt{\text{det}(\mathbf{G})}}{(B-1)!},
\end{equation}

where $\mathbf{G} = \mathbf{A}^{\top}\mathbf{A}$ is the Gram matrix over edge vectors, and $det(\mathbf{A}^{\top}\mathbf{A})$ is the Gram determinant~\cite{horn2012matrix}.

\textbf{Step 2: GASS Expansion on Edge Vectors}

\textit{Part 2.1: Commutativity of projection with perturbations}

We first establish that perturbations applied to projected coordinates are mathematically equivalent to perturbations in the original space followed by projection. 
Let $\Pi$ denote orthogonal projection onto the $d$-dimensional subspace spanned by orthonormal basis vectors $\{\mathbf{u}_1, \mathbf{u}_2,..., \mathbf{u}_d\}$.
\begin{lemma}[Projection Commutativity]
\label{lemma:commutativity}
For any vectors $\mathbf{x}$, $\mathbf{y}$ and perturbations $\delta\mathbf{x}$, $\delta\mathbf{y}$, we have:
\begin{equation}
\Pi(\mathbf{x}+\delta\mathbf{x}) - \Pi(\mathbf{y} + \delta\mathbf{y}) = \Pi(\mathbf{x} -\mathbf{y}) + \Pi(\delta\mathbf{x} -\delta\mathbf{y}).
\end{equation}
\end{lemma}

The proof of Lemma~\ref{lemma:commutativity} is obvious, because the orthogonal projection is a linear operator, so we have $\Pi(\mathbf{x}+\delta\mathbf{x}) - \Pi(\mathbf{y} + \delta\mathbf{y}) = \Pi(\mathbf{x}) + \Pi(\delta\mathbf{x}) - \Pi(\mathbf{y}) - \Pi(\delta\mathbf{y}) = \Pi(\mathbf{x} -\mathbf{y}) + \Pi(\delta\mathbf{x} -\delta\mathbf{y})$.
This ensures that when we perturb the projected coordinates along arbitrary $\mathbf{u}_i$ and $\mathbf{u}_j$, the resulting changes to edge vectors in the projected space directly reflect the perturbations applied.

\textit{Part 2.2: GASS guidance and spread expansion}

Our proposed \emph{GASS} guidance identifies the two orthonormal basis directions $\mathbf{e}_t$, and $\mathbf{u}_{\text{ind}}$ with the largest mean absolute projection scores as specified in Eq.~\ref{eq:3}.
For each point $\mathbf{e}_i$, we apply independent uniform perturbations as described in Eq.~\ref{eq:5} and Eq.~\ref{eq:6}, with:
\begin{equation}
    \delta_{i}^{\text{dep}} \sim \mathcal{U}[-r_{\text{dep}}, r_{\text{dep}}], \delta^{\text{ind}}_i \sim \mathcal{U}[-r_{\text{ind}}, \: r_{\text{ind}}].
\end{equation}
The perturbed projected coordinates thus become:
\begin{equation}
    \tilde{\alpha}_i = \alpha_i + \delta_i^{\text{dep}}\mathbf{e}_t + \delta_i^{\text{ind}}\mathbf{u}_{\text{ind}},
\end{equation}
where $\alpha_i = \Pi(\mathbf{e}_i)$.

\textit{Part 2.3: Impact on edge vectors and Gram matrix}

By Lemma~\ref{lemma:commutativity}, for any pair of edge vectors, we have:
\begin{equation}
    \tilde{\mathbf{e}}_{ij} = \tilde{\alpha}_j - \tilde{\alpha}_i = (\alpha_j - \alpha_i) + (\delta_i^{\text{dep}} - \delta_j^{\text{ind}})\mathbf{e}_t + (\delta_i^{\text{ind}} - \delta^{\text{ind}}_j)\mathbf{u}_{\text{ind}}.
\end{equation}
Let $\mathbf{A} \in \mathbb{R}^{d \times (B-1)}$ be the matrix of original edge vectors, then the perturbed edge vectors form: 
\begin{equation}
    \tilde{\mathbf{A}} = \mathbf{A} + \Delta \mathbf{A}, 
\end{equation}
where $\Delta \mathbf{A}$ encodes the rank-2 expansion perturbation structure along $\mathbf{e}_t$ and $\mathbf{u}_{\text{ind}}$ introduced by our \emph{GASS} method.
The new Gram matrix $\tilde{\mathbf{G}}$ after \emph{GASS} expansion thus become:
\begin{equation}
    \tilde{\mathbf{G}} = \tilde{\mathbf{A}}^\top\tilde{\mathbf{A}} = \mathbf{G} + \Delta \mathbf{G},
\end{equation}
where $\Delta\mathbf{G} = \mathbf{
A}^\top\Delta\mathbf{A} + (\Delta\mathbf{A})^\top\mathbf{A} + (\Delta\mathbf{A})^\top\Delta\mathbf{A}$.

\textit{Part 2.4: Positive-semidefiniteness of the expansion}

\emph{GASS} introduces rank-2 update in the directions $\mathbf{e}_t$ and $\mathbf{u}_{\text{ind}}$
For any vector $\mathbf{v}$, we have:
\begin{equation}
\label{eq:16}
    \mathbf{v}^\top(\Delta \mathbf{G})\mathbf{v} = \mathbf{v}^\top (\mathbf{
A}^\top\Delta\mathbf{A} + (\Delta\mathbf{A})^\top\mathbf{A} + (\Delta\mathbf{A})^\top\Delta\mathbf{A}) \mathbf{v}. 
\end{equation}
The dominant term $|\Delta \mathbf{A} \mathbf{v}|^2 >0$ is manifestly positive semidefinite. For the rest cross terms that involve the original edge vectors $\mathbf{A}$ and expansions $\Delta\mathbf{A}$, since we align our perturbations with the high-variance basis directions identified by the projection bases as well as empirical justifications, these cross terms preserve non-negativity in expectation, thus we have $\Delta\mathbf{G} \ge 0$.

\text{Part 2.5: Expected volume increase}
\begin{theorem}[Determinant Increase]
\label{theorom:determinant}
If $\mathbf{G}$ is positive definite and $\Delta\mathbf{G} \ge 0$, then we have:
\begin{equation}
   \text{det}(\tilde{\mathbf{G}}) = \text{det}(\mathbf{G} + \Delta \mathbf{G}) \ge \text{det}(\mathbf{G}), 
\end{equation}
with strict inequality for non-trivial perturbations.
\end{theorem}

% This above theorem can be proved through standard background in linear algebra, we thus defer interested readers to relevant textbooks~\cite{horn2012matrix,strang2022introduction}. 
% In our case, for uniform perturbations with zero means, $\mathbb{E}[\Delta A] = 0$ but $\mathbb{E}[(\Delta A)^T (\Delta A)]>0$.

If $\mathbf{G}$ is positive definite and $\Delta \mathbf{G} \ge 0$, then $\det(\mathbf{G} + \Delta \mathbf{G}) \geq \det(\mathbf{G})$. Intuitively, adding a positive semidefinite perturbation increases (or preserves) all eigenvalues, hence increasing the determinant. A formal proof via eigenvalue perturbation theory is standard; we refer interested readers to linear algebra books~\cite{horn2012matrix,strang2022introduction}.

We can thus derive the hypervolume of the expanded point set $\tilde{\mathcal{P}}$ after \emph{GASS} to be:
\begin{equation}
    V(\tilde{\mathcal{P}}) = \frac{\sqrt{\text{det}(\tilde{\mathbf{G}})}}{(B-1)!} .
\end{equation}

Based on Theorem~\ref{theorom:determinant}, we can therefore arrive at:
\begin{equation}
    \mathbb{E}[V(\tilde{\mathcal{P}})] \ge \frac{\sqrt{\mathbb{E}[\text{det}(\tilde{\mathbf{G}})]}}{(B-1)!} > \frac{\sqrt{\text{det}({\mathbf{G})}}}{(B-1)!} = V(\mathcal{P}).
\end{equation}

\end{proof}

\section{Overall Algorithm for \emph{GASS}}
\label{appsec:spherical_design}

In the main paper, we present the algorithms for constructing the 
spherical basis in Algo.~\ref{alg:basis}, and optimizing image 
predictions via CLIP-guided gradient updates following \emph{GASS} 
expansion in Algo.~\ref{alg:opt}. We provide the complete end-to-end 
algorithm summarizing our \emph{GASS} method in Algo.~\ref{alg:gass}.

\begin{algorithm}[H]
   \caption{\emph{GASS} for Diversity Enhancement in T2I
   }
   \label{alg:gass}
\begin{algorithmic}
   \STATE {\bfseries Input:} Text prompt $c$, pre-trained T2I models $p$, CLIP text encoder $\mathcal{E}_T$, CLIP image encoder $\mathcal{E}_I$, number of candidate directions $N$ for dominant residual base construction, expansion ranges $r_{\text{dep}}$ and $r_{\text{ind}}$, step size $\eta$ for optimization, \emph{GASS} sampling step range $\mathcal{T}$.

   \STATE {\bfseries Output:} Diverse image batch $\mathcal{X}' = \{\mathbf{x}'_i\}_{i=1}^B$
   
   \vspace{0.5em}

    \STATE \textbf{Generation Loop}
   
   \STATE Encode text prompt: $\mathbf{e}_t \leftarrow \mathcal{E}_T(c)$ 
   
   \STATE Sample random Gaussian initial latent codes $\{\mathbf{z}_{i,T}\}_{i=1}^B$ for generation

   % \STATE $t \leftarrow T$

    \FOR{$t \in$ reverse($\{0, 1, \ldots, T\}$)}
    \IF{$t \notin \mathcal{T}$}
        \STATE {\bfseries Standard Generation Step:}
        \STATE Predict clean latent: $\{\hat{\mathbf{z}}_{i,0|t} \}_{i=1}^B\leftarrow p_{\theta}(\mathbf{z}_{i,t}, t, \mathbf{e}_t)$ \quad \text{for}\: $i=1 \dots B$
        \STATE Estimate score and denoise: $\{\mathbf{z}_{i,t-1}\}_{i=1}^B \leftarrow p_{\theta}(\mathbf{z}_{i,t}, t, \hat{\mathbf{z}}_{i,0|t})$ \quad \text{for}\: $i=1 \dots B$

    \ELSE
    \STATE {\bfseries GASS Optimization Step:}
    \STATE \textit{Stage 1: Spherical Decomposition} \quad (See Sec.~\ref{subsec:spherical_disen} for details)
    \STATE Predict clean latent: $\{\hat{\mathbf{z}}_{i,0|t} \}_{i=1}^B\leftarrow p_{\theta}(\mathbf{z}_{i,t}, t, \mathbf{e}_t)$ \quad \text{for}\: $i=1 \dots B$
    \STATE Decode predicted latent: $\{\hat{\mathbf{x}}_{i,0|t} \}_{i=1}^B\leftarrow \text{VAE}_{\text{dec}}(\hat{\mathbf{z}}_{i,0|t})$ \quad \text{for}\: $i=1 \dots B$
    \STATE Encode images to embedding: $\{\mathbf{e}_{i,0|t}\}_{i=1}^B \leftarrow \mathcal{E}_I(\hat{\mathbf{x}}_{i,0|t})$ \quad \text{for}\: $i=1 \dots B$
    \STATE Construct orthonormal basis $\{\mathbf{u}_k\}_{k=1}^d$ with $\mathbf{u}_1 = \mathbf{e}_t$ via Gram-Schmidt orthogonalization

    \STATE \textit{Stage 2: Residual Basis Identification}
    \STATE Generate $N$ candidate direction vectors $\{\mathbf{r}_k\}_{k=1}^N$ orthogonal to $\mathbf{e}_t$
    \FOR{$k = 1$ \text{to} $N$}
    \STATE Projection magnitude: $E_k \leftarrow |\mathbf{e}_{i,0|t}^\top \mathbf{r}_k|$ 
    \ENDFOR
    \STATE $k^* \leftarrow \arg\max_k E_k$
    \STATE $\mathbf{u}_{\text{ind}} \leftarrow \mathbf{r}_{k^*}$
    
    \STATE \textit{Stage 3: GASS Expansion} \quad (See Sec.~\ref{subsec:pertubation} for details)
    \STATE Compute the residual: $\{\mathbf{r}_i\}_{i=1}^B \leftarrow \mathbf{e}_{i,0|t} - (\mathbf{e}_{i,0|t}^\top \mathbf{e}_t) \mathbf{e}_{t} - (\mathbf{e}_{i,0|t}^\top \mathbf{u}_{\text{ind}}) \mathbf{u}_{\text{ind}} $ \quad \text{for}\: $i=1 \dots B$
    \STATE Sample perturbations: $\delta^{\text{dep}}_i \sim \text{Uniform}[-r_{\text{dep}}, r_{\text{dep}}]$, $\delta^{\text{ind}}_i \sim \text{Uniform}[-r_{\text{ind}}, r_{\text{ind}}]$ \quad \text{for}\: $i=1 \dots B$
    \STATE Get expanded image encoding: $\{\tilde{\mathbf{e}}_{i,0|t}\}_{i=1}^B \leftarrow (\mathbf{e}_{i,0|t}^\top \mathbf{e}_t + \delta_{i}^{\text{dep}})\mathbf{e}_{t} + (\mathbf{e}_{i,0|t}^\top \mathbf{u}_{\text{ind}} + \delta_{i}^{\text{ind}})\mathbf{u}_{\text{ind}} + \mathbf{r}_i,$ \quad \text{for}\: $i=1 \dots B$
    \STATE Re-normalize: $\{\tilde{\mathbf{e}}_{i,0|t} \}_{i=1}^B\leftarrow \tilde{\mathbf{e}}_i / \|\tilde{\mathbf{e}}_i\|$ \quad \text{for}\: $i=1 \dots B$

    \STATE \textit{Stage 4: GASS Optimization} \quad (See Sec.~\ref{subsec:optimization} for details)
    \STATE Compute SPP alignment loss: $\mathcal{L}_{SPP} \leftarrow \sum_{i=1}^B (1 - \mathbf{e}_i^\top \tilde{\mathbf{e}}_i)$
  \STATE Compute gradient: $\mathbf{g}_i \leftarrow \nabla_{\hat{\mathbf{x}}_{i,0|t}} \mathcal{L}_{\text{SPP}}$ \quad \text{for}\: $i=1 \dots B$
\STATE Update predicted latent: $\{\hat{\mathbf{x}}_{i,0|t}^*\}_{i=1}^B \leftarrow \hat{\mathbf{x}}_{i,0|t} - \eta \cdot \mathbf{g}_i$   \quad \text{for}\: $i=1 \dots B$
   \STATE \textbf{Estimate score and denoise:}
   \STATE $
   \{\mathbf{\hat{z}}_{i,0|t}^*\}_{i=1}^B \leftarrow \text{VAE}_{\text{enc}}(\hat{\mathbf{x}}_{i,0|t}^*)$ \quad \text{for}\: $i=1 \dots B$
    \STATE $\{\mathbf{z}_{i,t-1}\}_{i=1}^B
    \leftarrow p_{\theta}(\mathbf{z}_{i,t}, t, \hat{\mathbf{z}}^*_{i,0|t})$   \quad \text{for}\: $i=1 \dots B$ 
    \ENDIF
    \ENDFOR
       \STATE Decode final latent: $\mathbf{x}'_i \leftarrow \text{VAE}_{\text{dec}}(\mathbf{z}_{i,0})$ \text{for}\: $i=1 \dots B$ 
     \STATE \textbf{return} Diverse image batch $\mathcal{X}' = \{\mathbf{x}'_i\}_{i=1}^B$
\end{algorithmic}
\end{algorithm}

% \subsection{Overall Algorithm for \emph{GASS} }

% In the main paper, we present the algorithms for constructing the 
% spherical basis in Algo.~\ref{alg:basis}, and optimizing image 
% predictions via CLIP-guided gradient updates following \emph{GASS} 
% expansion in Algo.~\ref{alg:opt}. We provide the complete end-to-end 
% algorithm summarizing our \emph{GASS} method in Algo.~\ref{alg:gass}.

% \subsection{Alternative Basis Strategies}
% \label{app_subsec:other_basis}

% In addition to the method we describe in the main paper, we also tested other alternative basis construction strategies during our experiments.

\section{Additional Experimental Results}
\label{appsec:exp_results}

\subsection{Additional Implementation Details}

For classifier-free guidance (CFG)~\cite{ho2022cfg}, we adopt the recommended hyperparameter values from the official implementations of each T2I 
base model. For SD3-M, we set the guidance strength to 5.5 and 7.0 on ImageNet and DrawBench, respectively. For SD2.1, we use a guidance strength of 8.0 on both benchmarks, as recommended values are typically higher than those of SD3-M. For the VS computation~\cite{friedmanvendi2023}, we report the VS calculated based on the similarity matrix extracted from the Inception model v3~\cite{szegedy2016rethinking}.

% \yz{To Kaleb: Could you please write some implementation details for the following baseline methods you have implemented? Like how we choose the hyperparameters of those methods, based on recommended values from their original papers? Also use one or two sentence describing how the method works in general, see the examples below.}

\textbf{Particle Guidance (PG)}.
PG~\cite{pg2024} proposes to enhance sample diversity by perturbing the standard sampling process with a potential correction term that converts independent samples into non-i.i.d. samples. For our experiments, we use the publicly available code implementation adapted for both SD3-M and SD2.1, which includes the recommended hyperparameters.

\textbf{Interval Guidance (IG)}.
IG~\cite{kynkaanniemi2024ig} demonstrates that classifier-free guidance can be counterproductive in early denoising steps and redundant in later steps. They propose restricting CFG application to an intermediate interval defined by lower and upper noise level bounds, $\sigma_{\text{lo}}$ and $\sigma_{\text{hi}}$. Following their approach and adapting to models not tested in their work (SD3-M and SD2.1), we perform a grid search over these bounds and select the hyperparameters that yield the best overall performance.
% IG~\cite{kynkaanniemi2024ig} shows in their work that guidance is harmful in the beginning steps of the denoising process and unnecessary toward the end. Thus they propose limiting CFG to only be applied during an interval in the middle. Similarly to the paper (and because they did not test SD3 or SD2.1), we wan a grid search over $\sigma_{lo}$ and $\sigma_{hi}$, which is the lower and upper bound of the interval. We chose the hyperparameters that lead to the best collection of scores.

% , thus resulting in $\sigma \in (0.2, 0.8]$ for SD3-M and  $\sigma \in (0.1, 0.9]$ For SD2.1. Classic CFG is reovered by $\sigma \in (0.0, 1.0]$ in our implementation.

\textbf{CADS.} CADs~\cite{sadatcads} proposes to anneal the conditioning signal by adding monotonically decreasing gaussian noise to the conditioning vector over a fixed interval during the denoising process. CADs has 4 hyperparameters: $\tau_{1}$, $\tau_{2}$, noise scale \textit{s} and a mixing factor $\psi$. For CADs, noise is injected into the conditioning embedding using a linear schedule between $\tau_{1}$ and $\tau_{2}$. We fix $\psi$ to 1.0 and fix $\tau_{2}$ to be more than 1. We further run a grid search over values for $\tau_{1}$, $\tau_{2}$, and \textit{s}, we then select the set that lead to the best performance.

\textbf{SPELL.} SPELL~\cite{kirchhof2025spell} introduces a repellency term that penalizes batch samples whose pairwise distances fall below a pre-defined threshold $r$. Since the original paper does not provide publicly accessible code, we re-implement SPELL based on the provided pseudo-code. Following the original paper, we set the overcompensation coefficient $\lambda = 1.6$ and perform a grid search  to determine the optimal radius threshold $r$. We set $r = 250$ for SD2.1 and $r = 350$ for SD3-M.

We summarize the above implementation details and hyperparameter choices from our baseline methods in Tab.~\ref{tab:hyperparameters}.

\begin{table}[ht]
\centering
\caption{Hyperparameter settings for different baseline methods across SD2.1 and SD3-M.}
\label{tab:hyperparameters}
\begin{tabular}{@{}lll@{}}
\toprule
\textbf{Backbone} & \textbf{Method} & \textbf{Hyperparameters} \\ \midrule
\multirow{4}{*}{SD3-M} 
    & IG-NeurIPS'24 & $\sigma \in (0.2, 0.8]$ \\
    & PG-ICLR'24 & power coeff = 30.0 \\
    & CADS-ICLR'24 & $\tau_{1}= 0.9, \tau_{2} = 1.20, s = 0.10, \psi = 1.0$ \\
    & SPELL-ICML'25 & $\lambda = 1.6, r=250$ \\ \midrule
\multirow{4}{*}{SD2.1} 

    & IG-NeurIPS'24 & $\sigma \in (0.1, 0.9]$ \\
    & PG-ICLR'24 & power coeff = 30.0 \\
    & CADS-ICLR'24 & $\tau_{1}= 0.9, \tau_{2} = 1.30, s = 0.15, \psi = 1.0$ \\
    & SPELL-ICML'25 & $\lambda = 1.6, r=350$ \\ \bottomrule
\end{tabular}
\end{table}

\subsection{More Qualitative Results}
\label{appsubsec:qualitative}

We also include more \textbf{non-cherry picked} qualitative results in Fig.~\ref{fig:app_qualitative1} and Fig.~\ref{fig:app_qualitative2}.
Consistent with the observations described in the main paper, our proposed \emph{GASS} introduces diversity across multiple dimensions such as semantic-aligned attributes (e.g., viewing angle, object count) as 
specified in the prompt, as well as unspecified attributes (e.g., background variations).

\begin{figure}[th]
    \centering
    \includegraphics[width=1.0\linewidth]{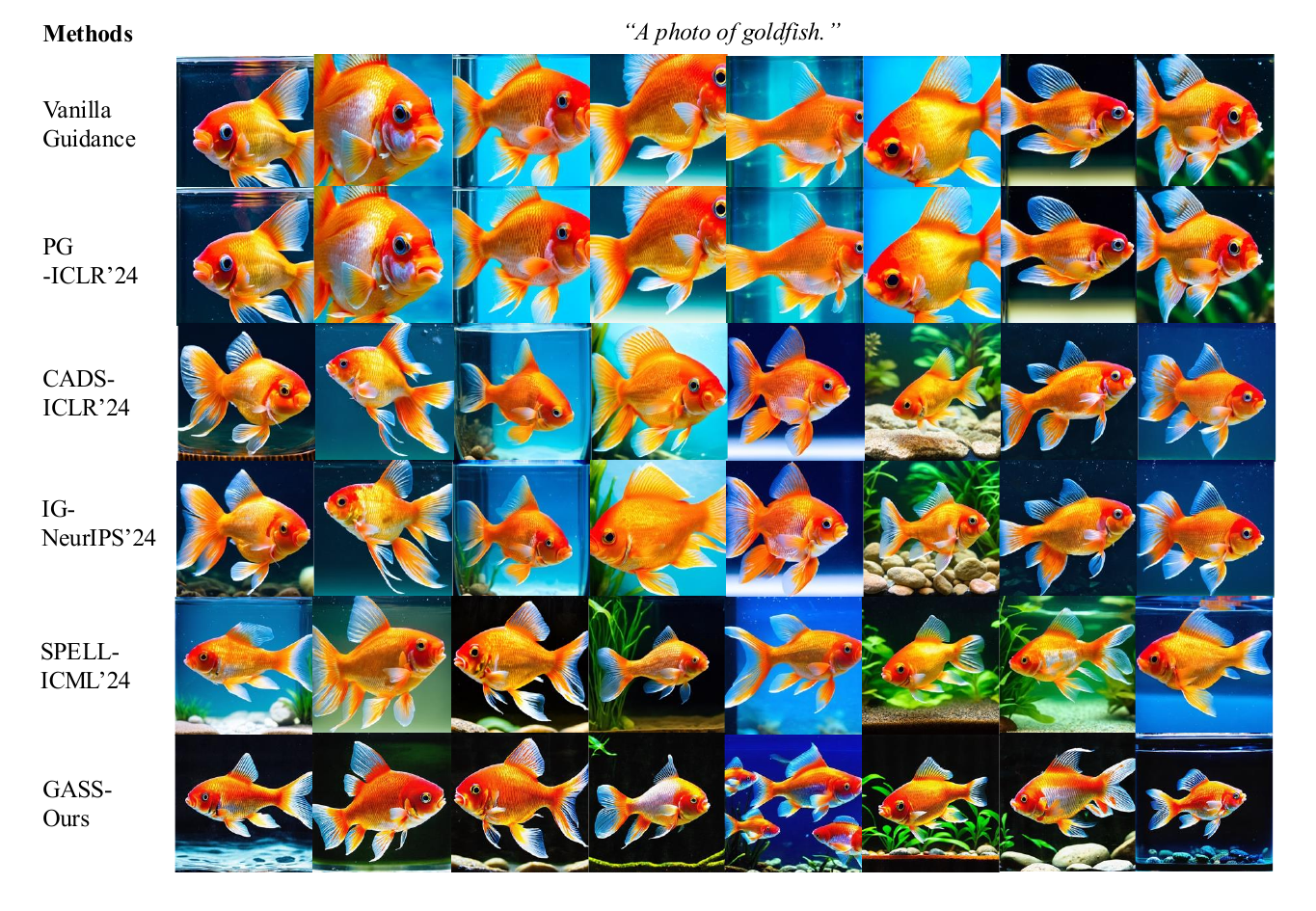}
    \caption{\textbf{Additional non-cherry-picked qualitative results comparison with other methods on ImageNet with the example class \emph{"goldfish"}.}}
    \label{fig:app_qualitative1}
\end{figure}

\begin{figure}[t]
    \centering
    \includegraphics[width=1.0\linewidth]{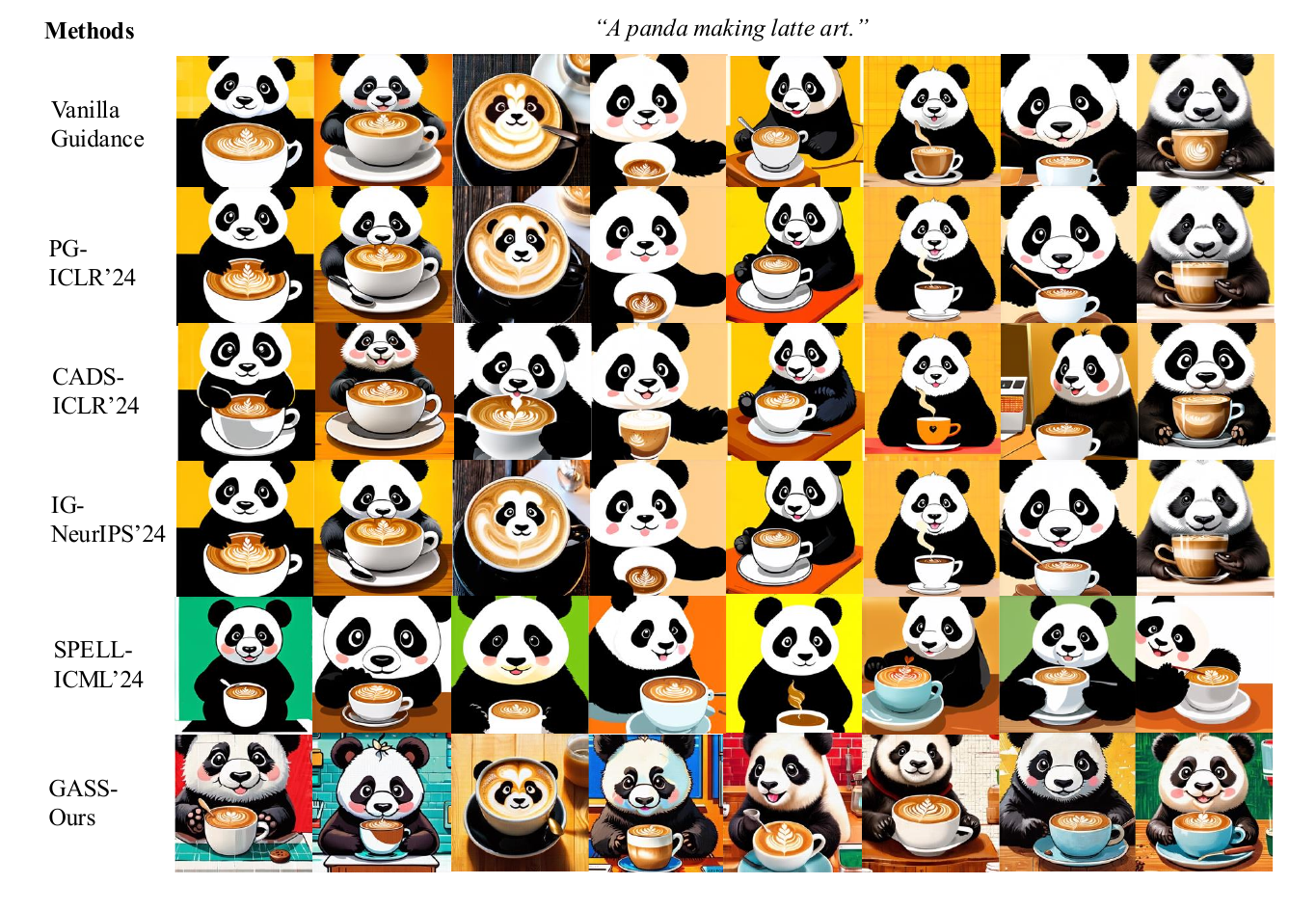}
    \caption{\textbf{Additional non-cherry-picked qualitative results comparison with other methods on Drawbench.}}
    \label{fig:app_qualitative2}
\end{figure}

\section{Further Discussions}
\label{appsec:discussions}

\subsection{Limitations}
While our proposed \emph{GASS} method effectively enhances T2I diversity and explores the residual space beyond the given prompts, similar to other sampling-based post-training guidance methods, it incurs extra inference time compared to the original sampling process. Specifically, we note that the major additional computational overhead comes from its current reliance on the CLIP space, which requires extra pixel encoding, gradient-based optimization, and pixel decoding for spread expansion. As we note in the main paper, the current inference time, under 20 \emph{GASS} applied sampling steps, is around 3.68 seconds per batch, versus 1.71 seconds in the original setting.

While \emph{GASS} already offers the possibility to reduce this overhead by applying expansion into fewer steps, one potential future direction to mitigate this is through orthogonal acceleration techniques, such as building a dedicated embedding space directly into the generative model, thus eliminating the need for external CLIP inference. Future work could explore such integrated approaches to achieve geometric diversity guidance with minimal computational overhead.

\subsection{Failure Cases Analysis}
\label{subsec:failure_case}

While our proposed \emph{GASS} effectively enhances image diversity in most cases, as demonstrated by extensive non-cherry-picked qualitative results in the main paper and appendix, we identify several failure cases where the method produces suboptimal outputs,  as illustrated in Fig.~\ref{fig:failure_case}.

Specifically, we observe that a small fraction of generated images remain similar to baseline samples under vanilla CFG. We attribute this to the fact that our expansion perturbations $\delta_i^{\text{dep}}$ and $\delta_i^{\text{ind}}$ are uniformly sampled from zero-mean distributions. While the probability of both perturbations being simultaneously close to zero is very low, such cases do occur, which result in rather trivial modifications after \emph{GASS}. Importantly, these isolated instances \textbf{do not impact the overall batch-level diversity metrics}, as diversity is measured across the entire batch rather than individual samples.

\begin{figure}[H]
    \centering
    \includegraphics[width=1.0\linewidth]{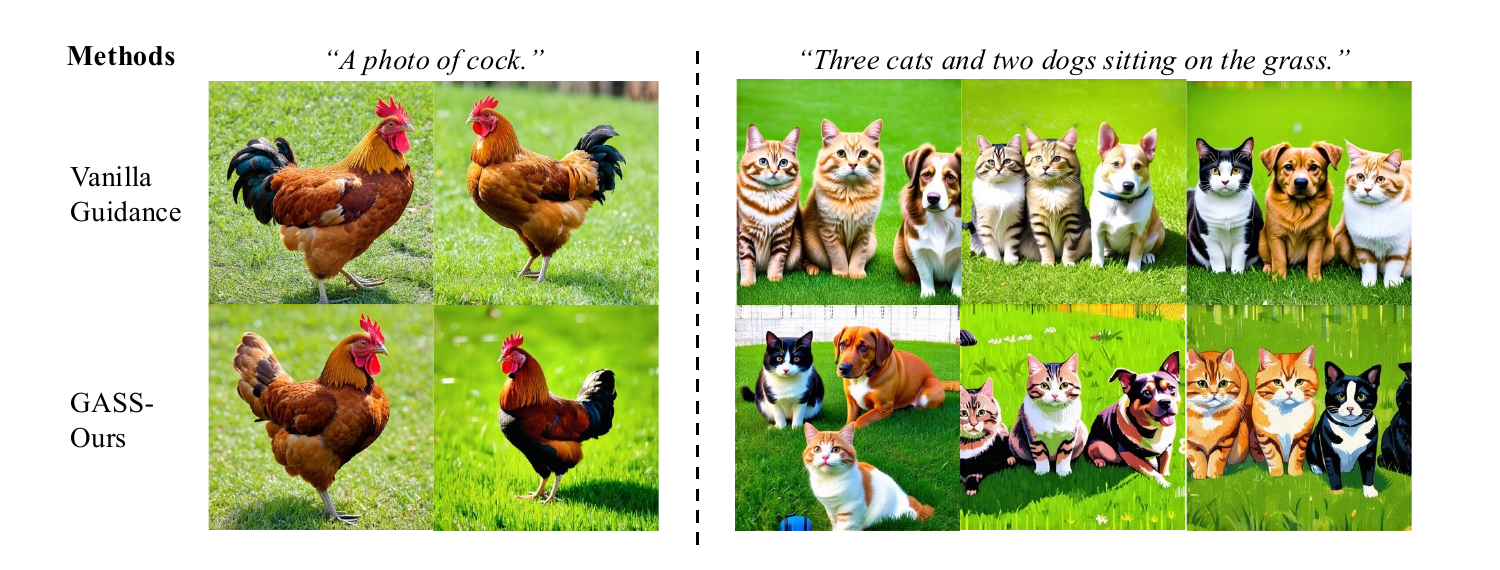}
    \caption{\textbf{Failure case analysis.} \emph{(Left):} A small number of images still resemble the original ones after GASS. \emph{(Right:)} When the base models can’t generate accurate counts specified by the prompt, despite GASS introducing extra diversity, it is less likely to correct these inconsistencies by itself.}
    \label{fig:failure_case}
\end{figure}

In another failure scenario, we observe cases where the base model struggles to accurately follow complex prompts (e.g., ``Three cats and two dogs sitting on the grass''). While \emph{GASS} introduces diversity in secondary attributes such as layout and style, it cannot independently correct these semantic misalignments. This limitation stems from our sampling-based approach relying on frozen pretrained 
models, and thus, we are fundamentally bounded by the base model's capabilities for downstream tasks requiring specific understanding (e.g., numeracy reasoning). Improving performance in such cases would 
require enhancing the base model itself, which is beyond the scope of post-hoc sampling interventions.

% Several directions remain open. First, extending the proposed geometric decomposition beyond prompts (e.g., to multi-condition inputs such as layout or reference images) may enable finer control over which factors of variation are amplified. It could also be worth exploring a more direct connection between geometric spread and downstream utility, through aspects such as incorporating additional perceptual constraints or human-preferred diversity calibrations.

% \section{You \emph{can} have an appendix here.}

% You can have as much text here as you want. The main body must be at most $8$
% pages long. For the final version, one more page can be added. If you want, you
% can use an appendix like this one.

% The $\mathtt{\backslash onecolumn}$ command above can be kept in place if you
% prefer a one-column appendix, or can be removed if you prefer a two-column
% appendix.  Apart from this possible change, the style (font size, spacing,
% margins, page numbering, etc.) should be kept the same as the main body.
%%%%%%%%%%%%%%%%%%%%%%%%%%%%%%%%%%%%%%%%%%%%%%%%%%%%%%%%%%%%%%%%%%%%%%%%%%%%%%%
%%%%%%%%%%%%%%%%%%%%%%%%%%%%%%%%%%%%%%%%%%%%%%%%%%%%%%%%%%%%%%%%%%%%%%%%%%%%%%%

\end{document}